\documentclass[lettersize,journal]{IEEEtran}
\usepackage{amsmath, amsfonts}
\usepackage{algorithm}
\usepackage{algpseudocode}
\usepackage{array}
\usepackage{float}
\usepackage[caption=false, font=normalsize, labelfont=sf, textfont=sf]{subfig}
\usepackage{textcomp}
\usepackage{float}
\usepackage{comment}
\usepackage{colortbl}
\usepackage{multirow}
\usepackage{stfloats}
\usepackage{amssymb}
\usepackage{hyperref}
\usepackage{lipsum}
\usepackage{url}
\usepackage{verbatim}
\usepackage{multicol}
\usepackage{graphicx}\usepackage{booktabs}
\usepackage{cite}
\usepackage{amsmath}
\usepackage{enumitem}
\usepackage{mathtools}
\usepackage[normalem]{ulem}
\usepackage{xcolor} 
\usepackage{amsthm}
\newtheorem{theorem}{Theorem}
\newtheorem{lemma}[theorem]{Lemma}
\newtheorem{definition}{Definition}
\usepackage{balance}

\algrenewcommand\algorithmicrequire{\textbf{Input:}}
\algrenewcommand\algorithmicensure{\textbf{Output:}}

\begin{document}

\title{Adaptive Clipping for Privacy-Preserving Few-Shot Learning: Enhancing Generalization with \\Limited Data}

\author{Kanishka Ranaweera~\IEEEmembership{Student Member, ~IEEE },  Dinh C. Nguyen,~\IEEEmembership{Member, ~IEEE, } Pubudu N. Pathirana,~\IEEEmembership{Senior Member, ~IEEE, } David Smith,~\IEEEmembership{Member, ~IEEE, } Ming Ding,~\IEEEmembership{Senior Member, ~IEEE, } Thierry Rakotoarivelo,~\IEEEmembership{Senior Member, ~IEEE, } Aruna Seneviratne,~\IEEEmembership{Senior Member, ~IEEE }

\IEEEcompsocitemizethanks{\IEEEcompsocthanksitem Kanishka Ranaweera is with School of Engineering and Built Environment, Deakin University, Waurn
Ponds, VIC 3216, Australia, and also with the Data61, CSIRO, Eveleigh, NSW 2015, Australia. \protect\\
E-mail: kranaweera@deakin.edu.au
\IEEEcompsocthanksitem Dinh C. Nguyen is with the Department of Electrical and Computer Engineering,  The University of Alabama in Huntsville Alabama, USA. \protect\\
E-mail: Dinh.Nguyen@uah.edu
\IEEEcompsocthanksitem Pubudu N. Pathirana is with School of Engineering and Built Environment, Deakin University, Waurn Ponds, VIC 3216, Australia. \protect\\
E-mail: pubudu.pathirana@deakin.edu.au
\IEEEcompsocthanksitem David Smith is with Data61, CSIRO, Eveleigh, NSW 2015, Australia. \protect\\
E-mail: david.smith@data61.csiro.au
\IEEEcompsocthanksitem Ming Ding is with Data61, CSIRO, Eveleigh, NSW 2015, Australia. \protect\\
E-mail: ming.ding@data61.csiro.au
\IEEEcompsocthanksitem Thierry Rakotoarivelo is with Data61, CSIRO, Eveleigh, NSW 2015, Australia. \protect\\
E-mail: thierry.rakotoarivelo@data61.csiro.au
\IEEEcompsocthanksitem Aruna Seneviratne is with School of Electrical Engineering and Telecommunications, University of New South Wales (UNSW), NSW, Australia. \protect\\
E-mail: a.seneviratne@unsw.edu.au

}
}

\markboth{}%
{}


\maketitle

\begin{abstract}
In the era of data-driven machine-learning applications, 
privacy concerns and the scarcity of labeled data have become paramount challenges. 
These challenges are particularly pronounced in the domain of few-shot learning, 
where the ability to learn from limited labeled data is crucial. 
Privacy-preserving few-shot learning algorithms have emerged as a promising solution to address such pronounced challenges.
However, 
it is well-known that privacy-preserving techniques often lead to a drop in utility due to the fundamental trade-off between data privacy and model performance. 
To enhance the utility of privacy-preserving few-shot learning methods, 
we introduce a novel approach called Meta-Clip. 
This technique is specifically designed for meta-learning algorithms, 
including Differentially Private (DP) model-agnostic meta-learning, DP-Reptile, and DP-MetaSGD algorithms, 
with the objective of balancing data privacy preservation with learning capacity maximization. 
By dynamically adjusting clipping thresholds during the training process, 
our Adaptive Clipping method provides fine-grained control over the disclosure of sensitive information, 
mitigating overfitting on small datasets and significantly improving the generalization performance of meta-learning models. 
Through comprehensive experiments on diverse benchmark datasets,
we demonstrate the effectiveness of our approach in minimizing utility degradation, 
showcasing a superior privacy-utility trade-off compared to existing privacy-preserving techniques. 
The adoption of Adaptive Clipping represents a substantial step forward in the field of privacy-preserving few-shot learning,
empowering the development of secure and accurate models for real-world applications, 
especially in scenarios where there are limited data availability.
\end{abstract}
  

\section{Introduction}
\IEEEPARstart{F}ew-shot learning has emerged as a challenging and important problem in the field of machine learning \cite{wang2020generalizing}. Traditional learning algorithms typically require a large amount of labeled data to achieve satisfactory performance. However, in many real-world scenarios, obtaining such extensive labeled data is often impractical or prohibitively expensive. Few-shot learning aims to address this limitation by enabling models to learn from a small number of labeled examples, mimicking the human ability to generalize knowledge from limited experience.

In the context of few-shot learning, meta-learning has gained considerable attention due to its ability to rapidly adapt to new tasks with limited data. Meta-learning algorithms leverage prior knowledge acquired from a distribution of related tasks to learn a meta-model capable of quickly adapting to new tasks. These algorithms hold promise for solving the few-shot learning problem by effectively utilizing shared information across tasks \cite{vilalta2002perspective,hospedales2021meta}.

While meta-learning algorithms have shown promising results, it is essential to acknowledge the pressing privacy concerns that have been highlighted in machine learning research. Several works have documented privacy violations and risks associated with handling sensitive data in the context of machine learning models \cite{swee,4531148,gambs2014anonymization}. These issues encompass various attacks, data inference vulnerabilities, and membership inference threats that can compromise user privacy. It is crucial to recognize that these privacy challenges faced in machine learning are also relevant to the domain of meta-learning, given the potential for handling sensitive information in the meta-learning process. 

\subsection{Related Work}

The exploration of privacy-preserving techniques in the realm of machine learning has witnessed a surge in response to growing concerns regarding data privacy and security. Among these techniques, differential privacy (DP) has emerged as a very important framework, offering robust guarantees for privacy preservation. The work by \cite{abadi2016deep, song2013stochastic, bassily2014private} introduced deep learning with DP, paving the way for privacy-aware machine learning models. This work laid the foundation for incorporating DP into various machine learning paradigms, including meta-learning.

In the context of meta-learning, Li et al. \cite{Li2020Differentially} explored DP meta-learning, emphasizing the need for privacy-preserving techniques in scenarios where meta-models are trained on sensitive or private datasets. Their work provides insights into how the challenges of few-shot learning and privacy concerns can be jointly addressed through the integration of DP.

Zhou et al. \cite{zhou2022tasklevel} investigated the task-level DP approach within the meta-learning paradigm. By introducing privacy preservation at the task level, this work offers a nuanced perspective on securing sensitive information during the meta-learning process. The task-level granularity allows for fine-tuned control over privacy guarantees, addressing concerns related to both individual tasks and the overall meta-model.

These advancements in incorporating DP into meta-learning underscore the ongoing efforts to reconcile the pursuit of improved machine learning performance with the imperative to safeguard user privacy in the face of increasing privacy challenges.


\subsection{Motivations and Key Contributions}

From our observation, traditional implementation of privacy-preserving techniques in meta-learning like DP often leads to a drop in utility, impacting the performance of the models. To mitigate this issue, we propose an adaptive clipping technique. This technique aims to strike a balance between privacy protection and model performance by dynamically adjusting the clipping threshold during the meta-learning process. By doing so, sensitive information can be effectively protected while minimizing the adverse effects on utility, thereby facilitating more efficient and privacy-aware meta-learning algorithms. 

By applying our adaptive clipping technique to well-known meta-learning algorithms, such as Model-Agnostic Meta-Learning (MAML) \cite{finn2017model}, Reptile \cite{nichol2018reptile}, and Meta-SGD \cite{li2017meta}, we demonstrate its significant potential to improve their performance while ensuring robust privacy guarantees. This paper aims to investigate the effectiveness of our adaptive clipping technique in enhancing the utility of task-level DP meta-learning algorithms for few-shot learning tasks. We focus on popular meta-learning algorithms, such as Model-Agnostic Meta-Learning (MAML) \cite{finn2017model}, Reptile \cite{nichol2018reptile}, and Meta-SGD\cite{li2017meta}, and demonstrate how our technique can improve their performance with privacy guarantees. The contributions of this work are highlighted as follows:

\IEEEpubidadjcol
 
\begin{enumerate}
    \item We propose an adaptive clipping technique that dynamically adjusts the clipping threshold during the training of DP meta-learning algorithms. Our technique effectively balances privacy preservation and utility, allowing the meta-model to leverage available information while meeting privacy constraints.
    
    \item We integrate the adaptive clipping technique into popular task-level DP meta-learning algorithms:Model-Agnostic Meta-Learning (MAML), Reptile, and Meta-SGD. Through extensive experiments on benchmark datasets, we demonstrate that our approach improves the performance and utility of these algorithms in few-shot learning scenarios.
    
    \item We are the first to investigate the application of task-level DP to the Meta-SGD algorithm, a widely used meta-learning algorithm. By incorporating DP into Meta-SGD, we contribute to the development of privacy-preserving techniques in meta-learning and establish its feasibility and effectiveness in the few-shot learning context.
    
    \item We provide a comprehensive convergence analysis of the adaptive clipping technique-enhanced meta-learning algorithms. Additionally, we conduct a thorough privacy analysis of our proposed techniques, quantifying the privacy guarantees provided by our DP meta-learning algorithms and analyzing the trade-offs between privacy and utility.
    
    \item We conduct comprehensive experiments to evaluate the performance of our approach. We compare our adaptive clipping technique-enhanced meta-learning algorithms against other state-of-art approaches, showcasing significant improvements in both privacy preservation and predictive performance.
    
\end{enumerate}

\subsection{Paper Organization}
The remainder of this paper is organized as follows: In Section 2, we provide a brief overview of few-shot learning, its challenges, and the role of meta-learning in addressing these challenges. We then discuss the importance of privacy preservation in meta-learning and introduce DP as a privacy-preserving framework. Section 3 presents the details of our proposed adaptive clipping technique and its integration into MAML, Reptile, and Meta-SGD. Subsequently, Section 4 describes the experimental setup and presents the results and analysis of our approach. Section 5 discusses the implications of our findings and potential future research directions. Finally, Section 6 concludes the paper and summarizes the contributions of our work.

\section{Background}
\label{sec:bg}
In this section, we provide background information on few-shot learning, meta-learning, DP learning, and adaptive clipping techniques.

\subsection{Few-Shot Learning}
Few-shot learning is a subfield of machine learning that tackles the problem of learning from limited labeled data. In traditional supervised learning, a large amount of labeled data is required to train models effectively. However, in many real-world scenarios, collecting abundant labeled examples is costly, time-consuming, or impractical. Few-shot learning addresses this challenge by aiming to learn from a small number of labeled instances, typically ranging from a few to a few hundred \cite{bengio2013optimization,hochreiter2001learning,wang2020generalizing,schmidhuber1992learning}.

The key objective of few-shot learning is to enable models to generalize well and adapt quickly to new tasks with only a limited number of examples. This requires the models to possess strong generalization abilities and the capability to learn efficiently from scarce labeled data.

Several approaches have been proposed in the literature to address few-shot learning. 
\begin{itemize}
    \item \textbf{Metric-based methods: }This centers on the idea of learning a similarity metric. This metric quantifies the similarity or dissimilarity between pairs of instances. These methods leverage the learned metric for tasks such as classification or retrieval. By considering how instances relate to each other, these methods can make informed predictions even when only a small labeled dataset is available.\cite{bromley1993signature,vinyals2016matching,snell2017prototypical,garcia2017few,sung2018learning}.
    \item \textbf{Meta-learning algorithms: } These techniques  to learn the learning process itself. These algorithms aim to capture underlying patterns or knowledge from multiple learning tasks. This meta-knowledge is then harnessed to facilitate swift adaptation to new tasks that come with limited examples. The strength of meta-learning lies in its ability to impart a sort of "learning to learn" capability, where the model becomes proficient at acquiring knowledge from few examples \cite{nichol2018reptile,finn2017model,li2017meta}.

    \item \textbf{Generative models: } This approach generate synthetic data samples within the target task's domain. By creating additional instances that are consistent with the characteristics of the given task, these models effectively expand the pool of labeled data. This augmentation aids in training more robust models, especially when the available labeled examples are insufficient \cite{zhang2018metagan,schonfeld2019generalized}.
\end{itemize}

Few-shot learning is a critical subfield that addresses the challenges posed by limited labeled data. By employing diverse methodologies, ranging from learning similarity metrics to meta-learning and generative modeling, this field empowers machine learning models to excel in scenarios where only a handful of labeled instances are accessible. This advancement is crucial for real-world applications that demand efficient and accurate learning even in data-scarce environments.

\subsection{Meta-Learning}
Meta-learning, also known as ``learning to learn", takes the idea of few-shot learning a step further. Instead of focusing on learning from a small number of examples for a specific task, meta-learning aims to develop algorithms that can learn from multiple learning tasks and generalize this knowledge to adapt quickly to new tasks.

Meta-learning algorithms learn a meta-model or meta-learner that captures the shared patterns or knowledge across different learning tasks. The meta-learner is trained on a distribution of tasks, each containing a few labeled examples. By leveraging the knowledge gained from these tasks, the meta-learner can quickly adapt to new tasks with limited labeled data. This makes meta-learning particularly useful in scenarios where obtaining labeled data for each specific task is challenging or time-consuming.

Several meta-learning algorithms have been proposed:
\begin{itemize}
    \item \textbf{Model-Agnostic Meta-Learning (MAML):} MAML \cite{finn2017model} is designed to find an initialization of model parameters that can be quickly adapted to new tasks with a small number of gradient updates. It optimizes for the meta-objective, making the model parameters suitable for rapid adaptation.
    \item \textbf{Reptile:} Reptile \cite{nichol2018reptile} is a simpler and more scalable approach than MAML. It performs multiple steps of stochastic gradient descent (SGD) on each task and then moves the initial parameters towards the parameters obtained after these updates.
    \item \textbf{Meta-SGD:} Meta-SGD \cite{li2017meta} extends MAML by not only learning the initial parameters but also learning the learning rates for each parameter. This allows for more flexibility and faster adaptation during the meta-learning process.
\end{itemize}

These algorithms differ in their specific approaches to capturing and leveraging meta-knowledge, but they all share the common goal of enabling models to generalize well and learn quickly from few examples.

The combination of few-shot learning and meta-learning has shown significant potential in a wide range of applications, including image recognition, natural language understanding, and reinforcement learning. By leveraging meta-learning techniques, models can effectively adapt to new tasks with minimal labeled data, making them highly applicable in real-world scenarios with limited resources\cite{singh2021metamed,huang2018natural,nagabandi2018learning,hsu2020meta}.

\subsection{Differential Privacy}
DP provides a rigorous framework for quantifying the privacy guarantees offered by machine learning algorithms\cite{dwork2006,dwork2014algorithmic}. It ensures that the presence or absence of an individual's data does not significantly impact the output of the algorithm. To understand DP, we introduce the following definitions and concepts:

\begin{definition}
    \textbf{(Neighboring Datasets):} Two datasets, $D$ and $D'$, are considered neighboring datasets if they differ by only a single data point.
\end{definition}

\begin{definition}
    \textbf{(Randomized Algorithm):} A randomized algorithm is an algorithm that introduces randomness into its computation or output.
\end{definition}

\begin{definition}
    \textbf{(($\varepsilon, \delta$)-DP):} A randomized algorithm $\mathcal{M}$ satisfies ($\varepsilon, \delta$)-DP if, for any two neighboring datasets $D$ and $D'$ and any subset $S$ of the algorithm's output space:

\begin{equation}
\label{DP_def1}
    Pr[\mathcal{M}(D)\in S] \leq e^{\varepsilon}Pr[\mathcal{M}(D')\in S] + \delta.
\end{equation}

\end{definition}

In this definition, $\varepsilon$ serves as a privacy parameter that controls the strength of privacy protection. A smaller $\varepsilon$ value indicates a stronger privacy guarantee, as it restricts the difference in the probabilities of any output event between neighboring datasets. The use of exponential function $e^{\varepsilon}$ ensures a multiplicative composition of privacy guarantees. The failure probability, denoted by the term $\delta$, is the largest amount by which the method is permitted to stray from the required privacy guarantee. For instance, in \ref{DP_def1}, we can achieve the $\varepsilon$ privacy guarantee with the probability of $(1 - \delta)$.  and $\delta$ is the likelihood of failure.  

To achieve DP, DP learning algorithms introduce controlled noise into the learning process. One commonly used mechanism is the addition of Gaussian noise to the model's parameters or gradients. The amount of noise added is determined by the sensitivity of the learning algorithm, which captures how much the algorithm's output changes in response to changes in the input dataset.

\begin{definition}
    \textbf{(Global Sensitivity):} The global sensitivity of a function $f$ quantifies the maximum amount by which $f$ can change when a single data point is added or removed from the dataset. Mathematically, the global sensitivity of function $f$ is given by:

\begin{equation}
    \triangle f = max_{D, D': |D - D'|_2 \leq 1} |f(D) - f(D')|_2
\end{equation}

\end{definition}

Here, $|.|_2$ represents the $L_2$ norm, which measures the Euclidean distance between vectors.

By calibrating the amount of noise added to the learning algorithm based on its global sensitivity and privacy parameters $\varepsilon$ and $\delta$, DP learning algorithms can achieve a balance between privacy and utility. Adaptive clipping is one such mechanism that dynamically adjusts the clipping threshold based on the sensitivity of individual gradients, effectively reducing noise and improving the utility-privacy trade-off.

DP learning has emerged as a principled approach to address privacy concerns in machine learning. It provides strong privacy guarantees by introducing carefully calibrated noise into the learning process, thereby obscuring individual-level information in the training data. DP ensures that the presence or absence of any particular training sample does not significantly impact the learned model or compromise the privacy of individual data contributors.

One popular technique for DP learning is DP stochastic gradient descent (DP-SGD)\cite{song2013stochastic,bassily2014private,abadi2016deep}. DP-SGD achieves privacy by bounding the sensitivity of the gradients computed on the training data and then adding noise proportional to this bound during the gradient updates. This can be formalized as follows:

Let $\mathcal{D}$ be the training dataset, and $L(\theta)$ be the loss function parameterized by $\theta$. The sensitivity of the loss function, denoted as $\Delta L$, measures the maximum change in $L(\theta)$ caused by removing or adding a single training sample. DP-SGD adds noise to the gradients $\nabla L(\theta)$ according to the privacy parameter $\epsilon$ to ensure DP:

\begin{equation}
   \theta' \gets \theta - \frac{\eta}{|\mathcal{D}|} \sum_{x \in \mathcal{D}} \left(\nabla L(\theta) + \mathcal{N}(0, \sigma^2I)\right) 
\end{equation}

where $\eta$ is the learning rate and $\mathcal{N}(0, \sigma^2I)$ represents Gaussian noise added to the gradients with variance $\sigma^2$. The noise magnitude is carefully controlled to achieve the desired privacy level while maintaining a reasonable trade-off between privacy and model accuracy.

\subsection{Privacy Preserving Few-Shot Learning}

In the context of few-shot learning, the scarcity of labeled data necessitates the utilization of supplementary information or knowledge to generalize effectively to new tasks or classes. This often implies the need for the incorporation of external data sources or the adoption of meta-learning frameworks, which can cause immediate privacy concerns. Potential privacy risks encompass unauthorized access to sensitive information, re-identification threats, and the deduction of private attributes or behaviors from trained models. Consequently, preserving the privacy of individual data points is pivotal to uphold ethical and responsible few-shot learning systems.

Numerous algorithms have been introduced to address privacy concerns in the domain of few-shot learning. These algorithms are designed to offer privacy assurances while upholding the usefulness and efficacy of the acquired models.

In their work, \cite{Li2020Differentially} introduce a novel DP algorithm designed for gradient-based parameter transfer techniques, such as those used in meta-learning. The paper formalizes the concept of task-global DP, provides provable transfer learning guarantees in convex settings, and illustrates substantial performance enhancements in tasks such as recurrent neural language modeling and image classification when compared to the more frequently explored local privacy approach.

\cite{zhou2022tasklevel} focuses on the challenges of meta-learning with task-level DP. This research addresses privacy concerns during the meta-learning training process, considering the exchange of task-specific information, which can pose privacy risks in sensitive applications. 

Applying DP learning to the few-shot learning setting presents several challenges. First, the limited availability of labeled examples per task exacerbates the trade-off between privacy and model performance. Privacy mechanisms that add noise to the gradients may lead to excessive noise accumulation, hindering the model's ability to generalize effectively. Moreover, fixed clipping norms used in traditional DP learning approaches may not be optimal for few-shot learning scenarios, where tasks can exhibit significant variations in their gradient magnitudes.

To address these challenges, we propose a novel technique called Meta-clip, which combines adaptive clipping with meta-learning in the context of DP few-shot learning. Meta-clip leverages the strengths of both adaptive clipping and meta-learning to achieve effective privacy preservation and improved generalization performance in few-shot learning tasks. The following sections provide a detailed explanation of the Meta-clip technique and present experimental evaluations to validate its effectiveness.

\begin{algorithm}[t!]
\caption{DP-MAML Algorithm w/ Meta-clip}\label{algmaml}
\begin{algorithmic}[1]
\Require Task distribution $\mathcal{D}$, Learning rate $\alpha$, Meta learning rate $\beta$, Meta-iterations $K$
\State Initialize model parameters $\theta$ 
\State Compute the initial clipping norm $\mathcal{C}$ as the median of the $\ell_2$ gradient norms in the first epoch
\For{$k=1$ \textbf{To} $K$}
    \For{each task $T_i$ sampled from $\mathcal{D}$}
        \State Sample a mini-batch of data from task $T_i$: 
        \State $\mathcal{D}_{T_i}=\{ (x_i, y_i) \}$
        \State Compute the loss on $\mathcal{D}_{T_i}$ using the current $\theta$: 
        \State $L_i(\theta) = \mathcal{L}(f_{\theta}(x_i), y_i)$
        \State Compute the gradient w.r.t. $\theta$ for task $T_i$: $\nabla_\theta L_i(\theta)$
        \State Clip gradient: 
        \State $\nabla_\theta L_i(\theta) = \nabla_\theta L_i(\theta) / max(1, \dfrac{||\nabla_\theta L_i(\theta)||_2}{\mathcal{C}})$
        \State Add noise:
        \State $\hat{\nabla}_\theta L_i(\theta)=\dfrac{1}{|T_i|}(\sum_{i}\nabla_\theta L_i(\theta)+\mathcal{N}(0,\sigma^2\mathcal{C}^2I))$
        \State Compute the updated parameter $\theta_i'$: 
        
        \State $\theta_i' = \theta - \alpha \hat{\nabla}_{\theta} L_i(\theta)$
    
    \EndFor
    \State Compute the meta-gradient over all tasks:

    \State $\nabla_{\theta_i}L_i(\theta'_{i}) = \sum_{(x_i, y_i) \in \mathcal{D}_{T_i}\sim\mathcal{D}} \mathcal{L}(f_{\theta'}(x_i), y_i)$
    \State Meta update: 
    \State $\theta, \mathcal{C} \leftarrow \theta, \mathcal{C} - \beta\nabla_{\theta_i}L_i(\theta'_{i})$
    \EndFor
    \Ensure $\theta_{t}$ and compute the overall privacy cost ($\varepsilon$, $\delta$) using a privacy accounting method.

\end{algorithmic}
\end{algorithm}

\section{Our Approach}\label{sec:approach}

In our methodology, we present an innovative adaptive clipping technique tailored for meta-learning algorithms, which we refer to as "Meta-clip". This technique is applied to three prominent few-shot learning algorithms, namely MAML, Reptile, and Meta-SGD. Our objective encompasses two key aspects: reducing utility loss and enhancing generalization, all while maintaining the same level of privacy assurances. 

\subsection{DP-MAML Algorithm with Meta-clip}
Algorithm \ref{algmaml} outlines our approach for incorporating Meta-clip within MAML. It takes four primary inputs: a task distribution \(\mathcal{D}\), learning rate \(\alpha\), meta-learning rate \(\beta\), and the number of meta-iterations \(K\). At the outset, the algorithm initializes model parameters \(\theta\) and a clipping threshold \(\mathcal{C}\). Notably, in the initialization of the clipping norm, the algorithm sets \(\mathcal{C}\) to the median of the per-sample gradient norms. In each meta-iteration \(k\) from 1 to \(K\), the algorithm iterates through tasks sampled from \(\mathcal{D}\). For each task \(T_i\), it samples a mini-batch of data \(\mathcal{D}_{T_i}\) and computes the loss \(L_i(\theta)\) using the current model parameters. The gradient \(\nabla_\theta L_i(\theta)\) is then computed and clipped to \(\mathcal{C}\) to control its magnitude. To protect user privacy, Gaussian noise is added to the clipped gradient, creating \(\hat{\nabla}_\theta L_i(\theta)\). The model parameters are updated using this noisy gradient. After processing all tasks, a meta-gradient is computed to capture overall task performance. Finally, a meta-update step adjusts the model parameters \(\theta\) and the clipping threshold \(\mathcal{C}\) using the meta-gradient. This iterative process repeats for \(K\) meta-iterations, allowing the model to adapt and learn from various tasks while preserving privacy through gradient clipping and noise addition. Our method is visually outlined in Figure~\ref{metaclip}.

\begin{algorithm}[t!]
\centering
\caption{DP Reptile Algorithm w/ Meta-clip}\label{reptile1}
\begin{algorithmic}
\Require Task distribution $\mathcal{D}$, learning rate $\alpha$, meta-learning rate $\beta$, meta-iterations $K$
\State Initialize model parameters $\theta$ randomly
\State Compute the initial clipping norm $\mathcal{C}$ as the median of the $\ell_2$ gradient norms in the first epoch
\For{$k=1$ \textbf{To} $K$}
\State Sample a meta-batch of tasks $\mathcal{T}_b \sim \mathcal{D}$
\For{each task $\mathcal{T}_i \in \mathcal{T}_b$}
\For{each $(x_i,y_i) \in \mathcal{T}_i$}
\State Compute the loss using the current $\theta$: 
\State $L_i(\theta) = \mathcal{L}(f_{\theta}(x_i), y_i)$
\State Compute the gradient w.r.t. $\theta$: $\nabla_\theta L_i(\theta)$
\State Clip gradient: 
\State $\nabla_\theta L_i(\theta) = \nabla_\theta L_i(\theta) / max(1, \dfrac{||\nabla_\theta L_i(\theta)||_2}{\mathcal{C}})$
\State Add noise:
\State $\hat{\nabla}_\theta L_i(\theta)=\dfrac{1}{|\mathcal{T}_i|}(\sum_{i}\nabla_\theta L_i(\theta)+\mathcal{N}(0,\sigma^2\mathcal{C}^2I))$
\State Compute:
\State $W_i=\theta - \alpha \hat{\nabla}_{\theta} L_i(\theta)$
\EndFor
\EndFor
\State Update the model parameters using the Reptile update rule:
\State $\theta_{t+1} \longleftarrow \theta_{t} + \beta\dfrac{1}{|\mathcal{T}_b|}\sum_{\mathcal{T}_i}(W_i -\theta_{t})$
\State Update the clipping norm:
\State $\mathcal{C} \longleftarrow \mathcal{C} + \beta\dfrac{1}{|\mathcal{T}_b|}\sum_{\mathcal{T}_i}(W_i -\theta_{t})$
\EndFor
\Ensure $\theta_{t}$ and compute the overall privacy cost ($\varepsilon$, $\delta$) using a privacy accounting method.
\end{algorithmic}
\end{algorithm}

\subsection{DP-Reptile Algorithm with Meta-clip}
Following the principles outlined in the DP-MAML with Meta-clip algorithm, we present Algorithm~\ref{reptile1}: the DP Reptile (DP-Reptile) algorithm with Meta-clip. DP-Reptile is tailored for meta-learning scenarios and integrates DP into the task-level Reptile optimization procedure. In each meta-iteration, DP-Reptile samples a meta-batch of tasks from the task distribution \(\mathcal{D}\). For each task in the batch, the algorithm computes the loss and gradient with respect to the model parameters \(\theta\). Similar to DP-MAML, DP-Reptile employs gradient clipping to control the magnitude of gradients and adds Gaussian noise to protect user privacy. The noisy gradients are then used to compute updated model parameters \(W_i\) for each task. DP-Reptile applies the Reptile update rule to adjust the model parameters \(\theta\) and also updates the clipping norm \(\mathcal{C}\) during each meta-iteration. The process repeats for the specified number of meta-iterations \(K\), allowing DP-Reptile to learn meta-level information while preserving DP. Finally, the algorithm returns the model parameters \(\theta_t\) and computes the overall privacy cost (\(\varepsilon\), \(\delta\)) using a privacy accounting method such as the moments accountant \cite{abadi2016deep}, the Renyi Differential Privacy (RDP) accountant \cite{wang2019subsampled} or the Gaussian Differential Privacy (GDP) accountant \cite{gopi2021numerical}, ensuring robust privacy guarantees.

It's essential to highlight the distinction between the Reptile algorithm and MAML, as both are fundamental approaches in the field of meta-learning. While DP-MAML with Meta-clip, as described earlier, focuses on adapting model parameters within a privacy-preserving context, Reptile follows a different meta-learning paradigm. Reptile aims at capturing fast adaptation by performing a form of "inner-loop" optimization on each task, adjusting the model parameters closer to the fine-tuning direction. In contrast, MAML focuses on finding an initialization of model parameters that can be fine-tuned quickly to adapt to new tasks, effectively optimizing for "meta-objective" learning. DP-Reptile with Meta-clip, as presented in Algorithm~\ref{reptile1}, extends the Reptile framework by incorporating DP, to address privacy concerns in the context of meta-learning. This distinction in optimization strategies and privacy considerations is crucial when choosing an appropriate algorithm for specific meta-learning scenarios and privacy requirements.

\subsection{DP-MetaSGD with Meta-clip}

Next, we introduce Algorithm~\ref{meta-sgd-dp}: the DP-Meta-SGD algorithm with Meta-Clip. DP-Meta-SGD is a powerful approach for meta-learning with privacy guarantees. Similar to the previously discussed DP-MAML and DP-Reptile algorithms, DP-Meta-SGD is designed to adapt model parameters across a range of tasks while preserving user privacy. In each meta-iteration, the algorithm samples a meta-batch of tasks from the task distribution \(\mathcal{D}\). For each task, it computes the loss and gradients with respect to the model parameters \(\theta\). Gradient clipping is applied to manage the gradient magnitudes, and DP is ensured through the addition of Gaussian noise to the gradients. The noisy gradients are used to calculate updated model parameters \(\theta_i'\) for each task. DP-Meta-SGD distinguishes itself by incorporating an adaptive learning rate (\(\alpha\)) in its meta-optimization process, allowing the learning rate to evolve during meta-learning. Additionally, it dynamically adjusts the clipping norm (\(\mathcal{C}\)) during each meta-iteration. The algorithm iterates for a specified number of meta-iterations \(K\) to facilitate meta-learning while preserving privacy. Finally, DP-Meta-SGD returns the model parameters \(\theta_t\) and computes the overall privacy cost (\(\varepsilon\), \(\delta\)) using a privacy accounting method, ensuring robust privacy protection in the meta-learning process.

 In each meta-iteration, the algorithm samples a meta-batch of tasks from the task distribution \(\mathcal{D}\). For each task, it computes the loss and gradients with respect to the model parameters \(\theta\). Gradient clipping is applied to manage the gradient magnitudes, and DP is ensured through the addition of Gaussian noise to the gradients. The noisy gradients are used to calculate updated model parameters \(W_i\) for each task. DP-Meta-SGD incorporates a meta-gradient computation over all tasks, allowing it to adapt the meta-parameters, the learning rate, and the clipping norm during each meta-iteration. The algorithm iterates for a specified number of meta-iterations \(K\) to facilitate meta-learning while preserving privacy. Finally, DP-Meta-SGD returns the model parameters \(\theta_t\) and computes the overall privacy cost (\(\varepsilon\), \(\delta\)) using a privacy accounting method, ensuring robust privacy protection in the meta-learning process.

This brings us to the distinctions between the DP Meta-SGD (DP-Meta-SGD) algorithm with Meta-Clip and its counterparts, DP-MAML and DP-Reptile. While all three algorithms prioritize DP through gradient clipping and noise addition, DP-Meta-SGD stands out with its unique optimization strategy. DP-Meta-SGD, following a stochastic gradient descent (SGD) approach, not only updates model parameters but also incorporates an adaptive learning rate (\(\alpha\)) and dynamically adjusts the clipping norm (\(\mathcal{C}\)) during each meta-iteration. DP-Meta-SGD's approach is particularly advantageous in scenarios where fine-tuning the learning process, including adjusting the learning rate and clipping norm, is crucial for effective meta-learning while ensuring stringent privacy preservation.

\begin{algorithm}[t!]
\caption{DP Meta-SGD Algorithm w/ Meta-clip}\label{meta-sgd-dp}
\begin{algorithmic}
\Require Task distribution $\mathcal{D}$, Meta learning rate $\beta$, Meta-iterations $K$
\State Initialize model parameters $\theta$, learning rate $\alpha$
\State Compute the initial clipping norm $\mathcal{C}$ as the median of the $\ell_2$ gradient norms in the first epoch
\For{$k=1$ \textbf{To} $K$}
\State Sample a meta-batch of tasks $\mathcal{T}_b \sim \mathcal{D}$
\For{each task $\mathcal{T}_i \in \mathcal{T}_b$}
\For{each $(x_i,y_i) \in \mathcal{T}_i$}
\State Compute the loss using the current $\theta$: 
\State $L_i(\theta) = \mathcal{L}(f_{\theta}(x_i), y_i)$
\State Compute the gradient w.r.t. $\theta$: $\nabla_\theta L_i(\theta)$
\State Clip gradient: 
\State $\nabla_\theta L_i(\theta) = \nabla_\theta L_i(\theta) / max(1, \dfrac{||\nabla_\theta L_i(\theta)||_2}{\mathcal{C}})$
\State Add noise:
\State $\hat{\nabla}_\theta L_i(\theta)=\dfrac{1}{|\mathcal{T}_i|}(\sum_{i}\nabla_\theta L_i(\theta)+\mathcal{N}(0,\sigma^2\mathcal{C}^2I))$
\State Compute the updated parameter $\theta_i'$: 
\State $\theta_i' = \theta - \alpha \hat{\nabla}_{\theta} L_i(\theta)$
\EndFor
\EndFor
\State Compute the meta-gradient over all tasks:
\State $\nabla_{\theta_i}L_i(\theta'_{i}) = \sum_{(x_i, y_i) \in \mathcal{D}_{T_i}\sim\mathcal{D}} \mathcal{L}(f_{\theta'}(x_i), y_i)$
    \State Update the meta-parameters: 
    \State $\theta \leftarrow \theta - \beta\nabla_{\theta_i}L_i(\theta'_{i})$
    \State Update the learning rate:
    \State $\alpha \leftarrow \alpha - \beta\nabla_{\theta_i}L_i(\theta'_{i})$
    \State Update the clipping norm:
    \State $\mathcal{C} \leftarrow \mathcal{C} - \beta\nabla_{\theta_i}L_i(\theta'_{i})$
    \EndFor
    \Ensure $\theta_{t}$ and compute the overall privacy cost ($\varepsilon$, $\delta$) using a privacy accounting method.
\end{algorithmic}
 \end{algorithm}


\section{Theoretical Analysis}
In this section, we undertake a comprehensive theoretical analysis of DP meta-learning algorithms, DP-MAML, DP-Reptile, and DP-Meta-SGD, all enhanced with the privacy-preserving Meta-Clip mechanism. We explore privacy and convergence, shedding light on the algorithms' guarantees and meta-learned model behavior. 

\begin{figure*}[!ht]
\centering
\includegraphics[width=0.7\textwidth]{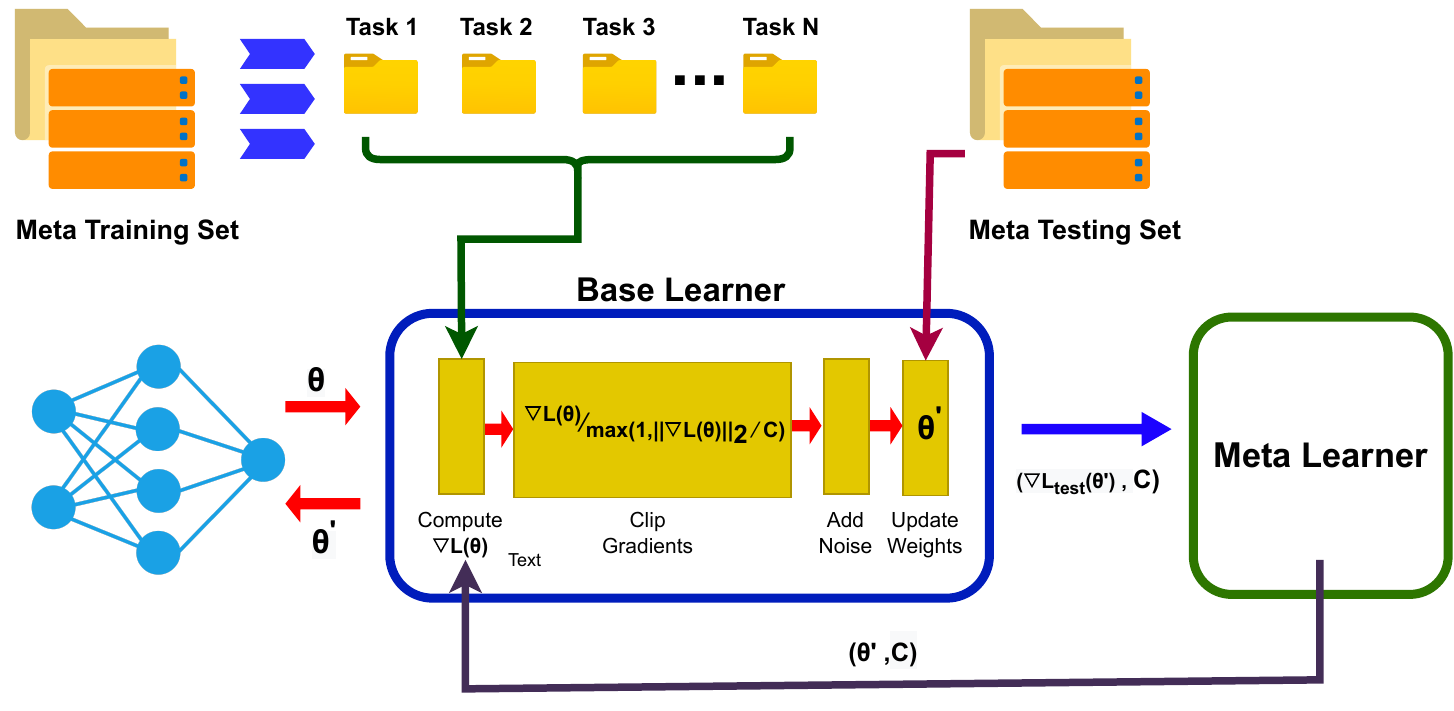}
\caption{llustration of the MAML algorithm enhanced with Meta-Clip.}
\label{metaclip}
\end{figure*}

\subsection{Privacy Analysis}

In this subsection, we present a unified privacy analysis for three meta-learning algorithms: DP-MAML, DP-Reptile, and DP-Meta-SGD, all enhanced with the Meta-Clip privacy-preserving approach. Leveraging the RDP Accountant, we rigorously assess the privacy guarantees achieved by these algorithms in terms of the privacy parameters $\varepsilon$ and $\delta$.

\subsubsection{Privacy Loss Random Variable}

We introduce the privacy loss random variable $\varepsilon_t$, representing the privacy loss incurred during the $t$-th iteration of DP-MAML, DP-Reptile, and DP-Meta-SGD with Meta-Clip. This variable quantifies the DP guarantees associated with individual iterations across all three algorithms.

\subsubsection{Rényi Differential Privacy}

Renyi Differential Privacy (RDP), an extension of standard DP, provides a family of privacy measures parameterized by the Rényi divergence parameter $\zeta \geq 0$. Given neighboring datasets $D_t$ and $D_t'$ and corresponding probability distributions $p_t$ and $p_t'$, the $\zeta$-Rényi divergence between $p_t$ and $p_t'$ is defined as:
\[
D_{\zeta}(p_t \| p_t') = \frac{1}{\zeta - 1} \log\left(\sum_{z} p_t(z)^{\zeta} p_t'(z)^{1 - \zeta}\right).
\]

\subsubsection{Rényi Differential Privacy Composition}

Utilizing RDP, we analyze the cumulative privacy guarantees achieved by DP-MAML, DP-Reptile, and DP-Meta-SGD with Meta-Clip in terms of privacy budget consumption. The cumulative privacy loss $\varepsilon_{\text{cumulative}}$ is computed by summing the Rényi divergences across all iterations:
\[
\varepsilon_{\text{cumulative}} = \sum_{t=1}^{T} D_{\zeta}(\mathcal{M}(D_t) \| \mathcal{M}(D_t')).
\]

\subsubsection{Privacy Parameters Computation} 

To compute the privacy parameters $\varepsilon$ and $\delta$, we leverage the RDP Accountant to estimate the RDP guarantees based on cumulative privacy loss $\varepsilon_{\text{cumulative}}$. The RDP Accountant offers a systematic methodology for deriving the privacy parameters using the Rényi divergence values, ensuring robust and verifiable DP guarantees across DP-MAML, DP-Reptile, and DP-Meta-SGD with Meta-Clip.

By adopting the RDP Accountant for privacy analysis, we establish that these three meta-learning algorithms, in conjunction with the Meta-Clip mechanism, provide strong and consistent DP guarantees. The privacy parameters $\varepsilon$ and $\delta$ enable a comprehensive evaluation of the trade-off between privacy preservation and utility, ensuring the stringent protection of individual data in the meta-learning.

\begin{figure*}[!h]
\centering
\includegraphics[width=0.9\textwidth]{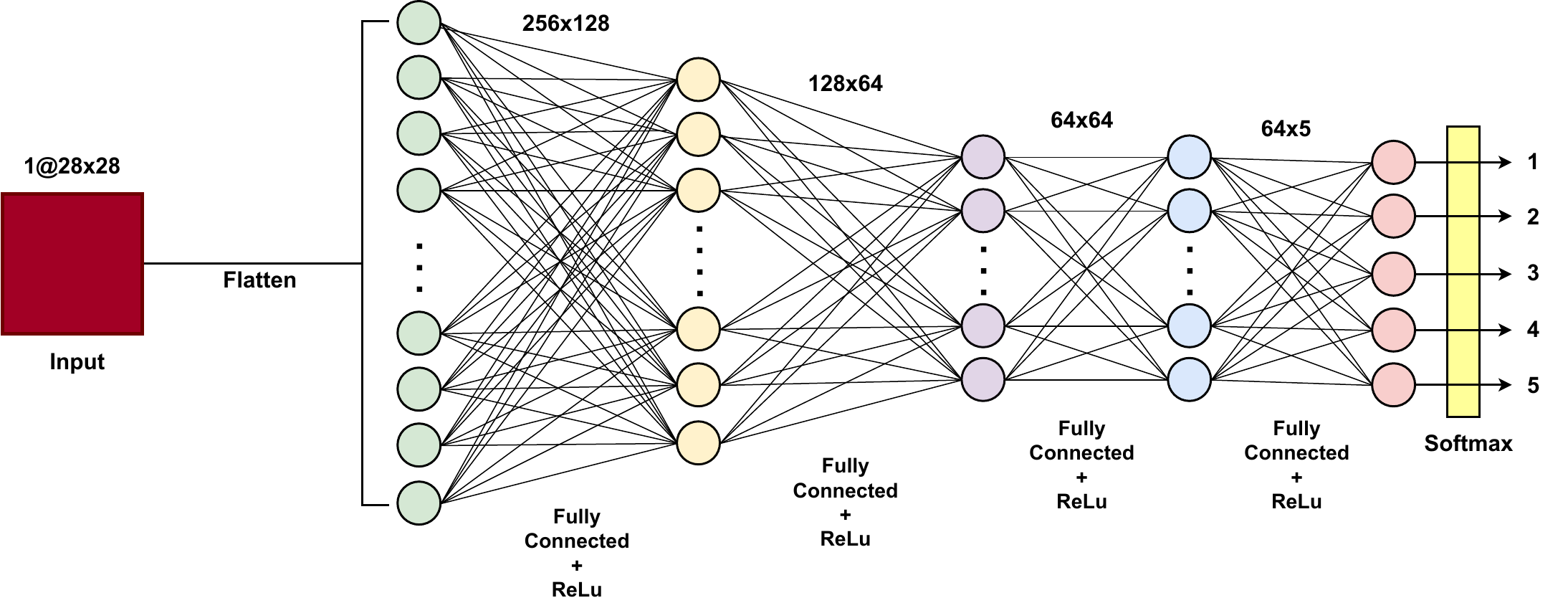}
\caption{Architecture of the Neural Network used to train on Omniglot dataset, consisting of multiple fully connected layers.}
\label{omni_cnn}
\end{figure*}

\begin{figure*}[!h]
\centering
\includegraphics[width=\textwidth]{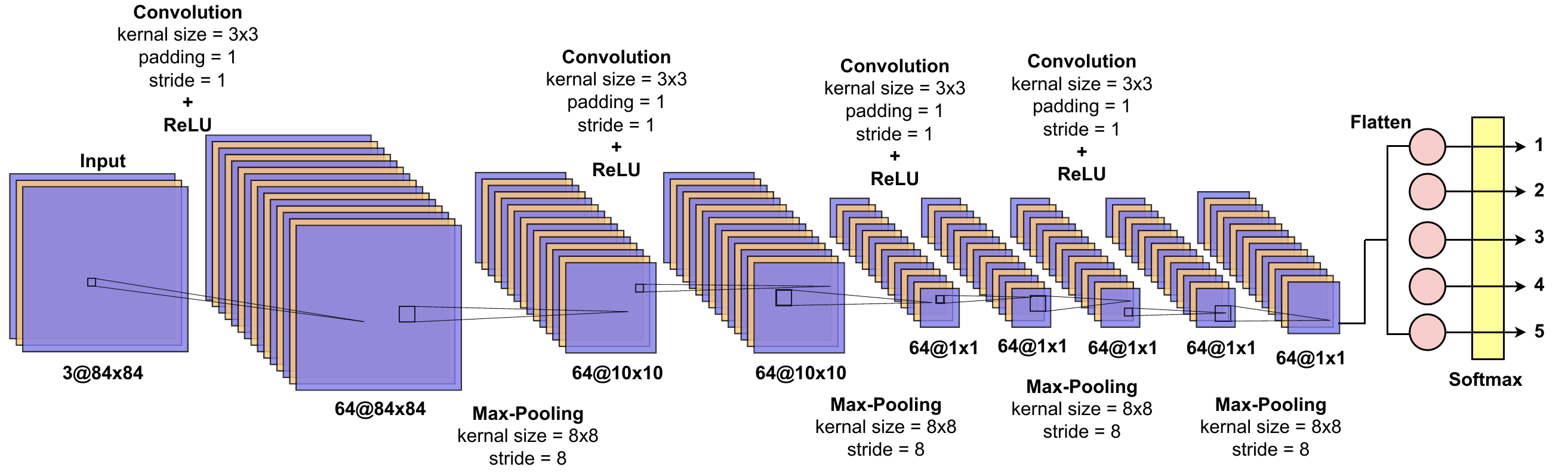}
\caption{Architecture of the Convolutional Neural Network (CNN) used to train on Mini Imagenet dataset, consisting of multiple convolutional layers followed by pooling layers.}
\label{miniim_cnn}
\end{figure*}


\subsection{Convergence Analysis}

This section focuses on the dynamic modification of the clipping norm and our theoretical analysis of the convergence qualities of the suggested DP-Meta learning algorithms in conjunction with Meta Clip. A unified convergence analysis can be applied to all three  DP meta-learning algorithms, including DP-MAML, DP-Reptile, and DP-MetaSGD, given their shared foundational principles and commonalities in the objective formulation. These algorithms typically aim to minimize the expected loss across tasks while ensuring differential privacy during the learning process.

\begin{equation}
\label{eq:m1}
    \min l(\theta) = \mathbb{E}_{k \in T_k}[l_k(\theta)].
\end{equation}

The loss function $l_k$ for task $k$ is typically specified in learning applications as a projected loss with regard to the probability distribution that produces the data for task $k$.

In order to keep things simple, we concentrate on determining an initialization $\theta$ that would result in a fair approximation for the minimizer of $l_k(\theta)$ after one gradient step of observation of a new task $k$. We can state this objective as;

\begin{equation}
\label{eq:mini}
   \min L(\theta) = \mathbb{E}_{k \in T_k}[L_k(\theta)] = \mathbb{E}_{k \in T_k}[l_k(\theta - \alpha\nabla l_k(\theta))],
\end{equation}

where $L_k(\theta) = l_k(\theta-\alpha\nabla l_k(\theta)).$

Essentially, the optimal solution for Problem \ref{eq:mini} is made to perform well on average when observing a task and examining the output after a single epoch. 

In order to simplify the calculations, we use the definition that follows:

\begin{definition}
We term a random vector $\theta{\mu}$ as a $\mu$-approximate first-order stationary point (FOSP) for Problem \ref{eq:mini} if it meets the condition $\mathbb{E}[\|\nabla L(\theta_\mu)\|] \leq \mu$.
\end{definition}

We seek to answer a crucial question in this section: Is it possible to find a $\mu$-FOSP for any given $\mu > 0$? If so, how many iterations are needed to reach this point? We first state our fundamental assumptions clearly before diving into these questions.

\begin{enumerate}
    \item The loss \(L(\cdot)\) is bounded, ensuring that \(\min F(\theta) > -\infty\), and defining \(\Delta = (L(\theta_0) - \min_{w\in\mathbb{R}^d})\).
    
    \item The functions \(L_k(\theta)\) are Lipschitz continuous and twice differentiable. This implies the existence of a constant \(\lambda_k\) such that \(\|\nabla l_k(\theta)-\nabla l_k(\theta')\| \le \lambda_k\|\theta-\theta'\|\). For ease of analysis, we often use \(\lambda = \max_k \lambda_k\) as a parameter representing the Lipschitz continuity of the gradients \(\nabla l_k\) for all \(k \in (0,K)\).
    
    \item \(\|\nabla^2 l_k(\theta)-\nabla^2 l_k(\theta')\| \leq \tau_k\|\theta-\theta'\|\) is the Hessian \(\nabla^2l_k\)-Lipschitz continuous for every \(k \in (0,K)\). Similar to the last point, for all \(k \in (0,K)\), we frequently utilise \(\tau = \max_k \tau_k\) as a parameter for the Lipschitz continuity of the gradients \(\nabla l_k\).
    
    \item The gradient \(\nabla l_k(\theta)\) exhibits bounded variance, ensuring that \(\mathbb{E}_i\|\nabla l(\theta) - \nabla l_k(\theta)\|^2 \leq \varphi^2\), where \(l(\theta) = \mathbb{E}_{k \in T_k}\nabla l_k(\theta)\).
    
    \item The stochastic gradients \(\nabla l_k(\theta)\) and Hessians \(\nabla^2 l_k(\theta)\) have bounded variance for all \(k\) and any \(\theta \in \mathbb{R}^d\). In particular, 
    \begin{equation}
        \mathbb{E}[\|\nabla l_k(\theta)-\nabla l_k(\theta)\|^2] \leq \widehat{\varphi}^2,
    \end{equation}

    \begin{equation}
        \mathbb{E}[\|\nabla^2 l_k(\theta)-\nabla^2 l_k(\theta)\|^2] \leq \varphi^2_H.
    \end{equation}
    Where, \(\widehat{\varphi}, \varphi_H > 0\).
\end{enumerate}

First, note that the Lipschitz continuity assumption \(\lambda_k\) also implies that the following inequality applies for every \(\theta, \theta' \in \mathbb{R}^d\) given our assumption that the functions \(l_k\) are twice differentiable:

\begin{equation}
\label{eq:90}
    -\lambda_k I_d \leq \nabla^2l_k(\theta) \leq \lambda_k\lambda_d. 
\end{equation}

Additionally, to ensure task-level privacy in individual client training, we use the following notation to denote the insertion of DP noise for simplicity's sake:
\begin{equation}
\nabla l_k(\theta) = \nabla l_k(\theta) + \mathcal{N}(0, \alpha^2\mathcal{C}^2I)
\end{equation}

In the DP-Meta learning algorithm, stochastic gradient descent minimises the global loss function \(L\), which may not be smooth over \(\mathbb{R}^d\) and whose smoothness parameter may be infinite. We formally state the parameter characterising the Lipschitz continuity of the gradients \(\nabla L\) in the following lemma.

\begin{lemma}
\label{lem:1}
    Let \(L\) be the objective function defined in \ref{eq:mini}, assuming \(\alpha \in [0, \frac{1}{\lambda}]\). For any \(\theta, \theta' \in \mathbb{R}_d\), the following inequality holds:

    \begin{equation}
        \|\nabla L(\theta) - \nabla L(\theta')\| \leq \min[\lambda(\theta),\lambda(\theta')]\|\theta-\theta'\|,
    \end{equation}

    where \(\lambda(\theta) = 4\lambda + 2\tau\alpha \mathbb{E}_{k \in T_k} \|\nabla l_k(\theta)\|\).
\end{lemma}

\begin{proof}
    
See Appendix~\ref{sec:app1}

\end{proof}

The objective function $L$ is shown to be smooth by the previously stated lemma, and this smoothness parameter depends on the predicted gradient norm. To put it another way, the smoothness parameter is affected by $\min[\mathbb{E}_{k \in T_k} \|\nabla l_k(\theta)\|, \mathbb{E}_{k \in T_k} \|\nabla l_k(\theta')\|]$ when evaluating the smoothness of gradients of $\theta$ and $\theta'$.

\begin{lemma}
For the case where $\alpha \in (0, \dfrac{1}{\lambda})$, let's consider $L$ in \ref{eq:mini}. Then, the following holds:
\label{lem:2}
\begin{equation}
\begin{aligned}
\label{eq:55}
    &\mathbb{E}_{\mathcal{D}_{s}\mathcal{D}_q}[\widehat{\nabla}l_k(\theta_t-\alpha\widehat{\nabla}l_k(\theta_t,\mathcal{D}^k_{s}),\mathcal{D}_q^k)| \mathcal{F}_t] \\&= \nabla l_k(\theta_t-\alpha\nabla l_k(\theta_t)) + e_{k,t},
\end{aligned}
\end{equation}

where, $\|e_{k,t}\| \leq \dfrac{\alpha\lambda\widehat{\alpha}}{\sqrt{\mathcal{D}_{s}}}$.

Furthermore, for any arbitrary $\phi > 0$, we get

\begin{equation}
\label{eq:56}
\begin{aligned}
    &\mathbb{E}_{\mathcal{D}_{s}\mathcal{D}_q}[\|\widehat{\nabla}l_k(\theta_t-\alpha\widehat{\nabla}l_k(\theta_t,\mathcal{D}^k_{s}),\mathcal{D}_q^k)\|^2| \mathcal{F}_t] \\& \leq (1+\dfrac{1}{\phi})\|\nabla l_k(\theta_t-\alpha\nabla l_k(\theta_t))\|^2 + \dfrac{(1+\phi)\alpha^2\lambda^2\widehat{\alpha}^2}{\mathcal{D}_{s}} + \dfrac{\widehat{\alpha}}{\mathcal{D}_q}.
\end{aligned}
\end{equation}
\end{lemma}

\begin{proof}
See appendix~\ref{sec:app2}
\end{proof}

Lemma~\ref{lem:2} provides an explanation of why $g_k$, the descent direction of DP-Meta learning, functions as a biased estimator of $\nabla L(\theta_t)$. It implies that the bias is limited by a constant that depends on the variance of stochastic gradients $\widehat{\nabla}l_k$ and the stepsize $\alpha$ for the inner steps.  As suggested by our result in \ref{eq:55}, the vector $\widehat{\nabla}l_k(\theta_t-\alpha \widehat{\nabla}l_k(\theta_t,\mathcal{D}^k_{s}),\mathcal{D}_q^k)$ becomes an unbiased estimate of $\nabla l_k(\theta_t-\alpha\widehat{\nabla}l_k(\theta_t,\mathcal{D}^k_{s})$ when $\alpha$ is set to 0. Further, the outcome presented in \ref{eq:56} indicates that the second moment of $\widehat{\nabla}l_k(\theta_t-\alpha\widehat{\nabla}l_k(\theta_t,\mathcal{D}^k_{s}),\mathcal{D}_q^k)$ is bounded above by the product of a factor involving $\|\nabla l_k(\theta_t-\alpha\widehat{\nabla}l_k(\theta_t))\|^2$ and a factor involving $\widehat{\alpha}^2$.

We offer the following to clarify the interactions between the gradients and provide an understanding of how they should behave.

\begin{lemma}
\label{lem:3}

Examine the formulations of $l(\cdot)$ in \ref{eq:m1} and $L(\cdot)$ in \ref{eq:mini} under the condition $\alpha \in [0, \dfrac{\sqrt{2}-1}{\lambda})$. Subsequently, for any $\theta \in \mathcal{R}^d$, we observe

    \begin{equation}
    \label{eq:m8}
        \|\nabla l(\theta)\|\leq O_1\|\nabla L(\theta)\|+O_2\varphi,
    \end{equation}
    \begin{equation}
    \begin{aligned}
    \label{eq:m9}
        \mathbb{E}_{k \in T_k}[\|\nabla L_k(\theta)\|^2]&\leq 2(1+\alpha\lambda)^2O_1^2\|\nabla L(\theta)\|^2\\&+(1+\alpha\lambda)^2(2O_2^2+1)\varphi^2,
    \end{aligned}
    \end{equation}

    where,

    $O_1=\dfrac{1}{1-2\alpha\lambda-\alpha^2\lambda^2},  O_2=\dfrac{2\alpha\lambda+\alpha^2\lambda^2}{1-2\alpha\lambda-\alpha^2\lambda^2}$.
\end{lemma}

\begin{proof}
   See Appendix~\ref{sec:app3}
 \end{proof}

Next, we look into analyzing the overall complexity of DP-Meta learning in its attempt to derive an $\mu$-FOSP for the loss functions $L$ outlined in \ref{eq:mini}.

\begin{lemma}
Consider L in \ref{eq:mini} for the case that $\alpha \in (0, \dfrac{1}{10\lambda}]$. Consider running DP-Meta learning with batch sizes satisfying $D_h \leq [2\alpha^2\varphi^2_H]$ and $T_k \leq 20$. Since $\widehat{\beta}(\theta)$, let $\beta_t = \widehat{\beta}(\theta_t)/18$. Next, DP-Meta learning determines a solution $\theta_{\mu}$ for any $\mu > 0$, such that
\label{lem:4}

\begin{equation}
\label{eq:lem2}
\begin{aligned}
    \mathbb{E}&[\|\nabla L(\theta_\mu)\|]\\&\geq \max\{\sqrt{14(1+\dfrac{\tau\alpha}{\lambda}\varphi)(\varphi^2(\dfrac{1}{T_k}+20\alpha^2\lambda^2)+\dfrac{\widehat{\varphi}^2}{T_k\mathcal{D}_q}+\dfrac{\widehat{\varphi}^2}{\mathcal{D}_{s}})} \\& ,\dfrac{14\tau\alpha}{\lambda}(\alpha^2(\dfrac{1}{T_k}+20\alpha^2\lambda^2)+\dfrac{\widehat{\varphi}^2}{T_k\mathcal{D}_q}+\dfrac{\widehat{\varphi}^2}{T_k\mathcal{D}_{s}}),\mu \},
\end{aligned}
\end{equation}

following maximum iterations of

 \begin{equation}
 \begin{aligned}
 \label{eq:lem3}
     &\mathcal{O}(1)\triangle \min\{\dfrac{\lambda+\tau\alpha(\varphi+\mu)}{\mu^2},\dfrac{\lambda}{\varphi^2(\dfrac{1}{T_k}+20\alpha^2\lambda^2)}\\&+\dfrac{\lambda(T_k\mathcal{D}_q+\mathcal{D}_{s})}{\widehat{\varphi}^2}\}
 \end{aligned}
 \end{equation}

iterations, where $\triangle=(L(\theta_0)-\min_{\theta\in \mathcal{R}^d}L(\theta))$.

\end{lemma}

\begin{proof}
See Appendix~\ref{sec:app4}
\end{proof}

Lemma~\ref{lem:4}'s result shows that we can locate a point $\theta^*$ after \ref{eq:lem3} iterations of DP-Meta learning. At this point, $\mathbb{E}[\|\nabla L(\theta^*)\|]$, the anticipated norm of its gradient, is bounded as shown in \ref{eq:lem2}. This important result suggests that by selecting appropriate hyperparameters that are proportionate to the inverse of $\mu$, a $\mu$-FOSP for problem \ref{eq:mini} can be obtained in a polynomial number of iterations. This also provides insights into the behavior of DP-Meta learning algorithms based on the size of the gradient norm. Specifically, it shows the impact of the gradient norm's magnitude, considering both large and small values. The expectation of the gradient norm is characterized in terms of various factors, including the regularization parameter, step size, and other constants. The analysis suggests that the convergence rate depends on the interplay of these parameters, with a careful balance needed for efficient convergence. Both upper and lower bounds are provided, revealing a trade-off between different terms in the convergence expression. This meticulous examination highlights the intricate relationship between the gradient norm's size and the convergence behavior of DP-Meta learning algorithms.

\section{Experimental Evaluation}

\begin{table*}[]
\caption{Training accuracies for MAML algorithm against two benchmarking datasets for both 1-shot and 5-shot training configurations. Here $\delta=10^{-5}$.}
\label{tab:maml}
\begin{tabular}{|c|clllll|clllll|}
\hline
                                                                                    & \multicolumn{6}{c|}{\cellcolor[HTML]{FFFFFF}{\color[HTML]{000000} OMNIGLOT}}                                                                                                                                                                                                                                                            & \multicolumn{6}{c|}{\cellcolor[HTML]{FFFFFF}{\color[HTML]{000000} MINI Imagenet}}                                                                                                                                                                                                                                                       \\ \cline{2-13} 
\multirow{-2}{*}{\begin{tabular}[c]{@{}c@{}}Optimization \\ Algorithm\end{tabular}} & \multicolumn{3}{c|}{1 Shot Accuracy}                                                                                                                               & \multicolumn{3}{c|}{5 Shot Accuracy}                                                                                                                               & \multicolumn{3}{c|}{1 Shot Accuracy}                                                                                                                               & \multicolumn{3}{c|}{5 Shot Accuracy}                                                                                                                               \\ \hline
\begin{tabular}[c]{@{}c@{}}Non Private \\ MAML\end{tabular}                         & \multicolumn{3}{c|}{98.7\%}                                                                                                                                        & \multicolumn{3}{c|}{99.9\%}                                                                                                                                        & \multicolumn{3}{c|}{48.7\%}                                                                                                                                        & \multicolumn{3}{c|}{63.11\%}                                                                                                                                       \\ \hline
                                                                                    & \multicolumn{3}{c|}{\cellcolor[HTML]{C0C0C0}\begin{tabular}[c]{@{}c@{}}DP property \\ $\varepsilon$\end{tabular}}                                                          & \multicolumn{3}{c|}{\cellcolor[HTML]{C0C0C0}\begin{tabular}[c]{@{}c@{}}DP property \\ $\varepsilon$\end{tabular}}                                                          & \multicolumn{3}{c|}{\cellcolor[HTML]{C0C0C0}\begin{tabular}[c]{@{}c@{}}DP property \\ $\varepsilon$\end{tabular}}                                                          & \multicolumn{3}{c|}{\cellcolor[HTML]{C0C0C0}\begin{tabular}[c]{@{}c@{}}DP property \\ $\varepsilon$\end{tabular}}                                                          \\ \cline{2-13} 
\multirow{-2}{*}{}                                                                  & \multicolumn{1}{c|}{\cellcolor[HTML]{C0C0C0}9.02} & \multicolumn{1}{c|}{\cellcolor[HTML]{C0C0C0}10.6} & \multicolumn{1}{c|}{\cellcolor[HTML]{C0C0C0}12.8} & \multicolumn{1}{c|}{\cellcolor[HTML]{C0C0C0}8.21} & \multicolumn{1}{c|}{\cellcolor[HTML]{C0C0C0}10.8} & \multicolumn{1}{c|}{\cellcolor[HTML]{C0C0C0}15.4} & \multicolumn{1}{c|}{\cellcolor[HTML]{C0C0C0}9.02} & \multicolumn{1}{c|}{\cellcolor[HTML]{C0C0C0}10.6} & \multicolumn{1}{c|}{\cellcolor[HTML]{C0C0C0}12.8} & \multicolumn{1}{c|}{\cellcolor[HTML]{C0C0C0}8.21} & \multicolumn{1}{c|}{\cellcolor[HTML]{C0C0C0}10.8} & \multicolumn{1}{c|}{\cellcolor[HTML]{C0C0C0}15.4} \\ \hline
\begin{tabular}[c]{@{}c@{}}Vanilla \\ DP-MAML\end{tabular}                          & \multicolumn{1}{l|}{76.04\%}                        & \multicolumn{1}{l|}{80.66\%}                         & \multicolumn{1}{l|}{81.40\%}                          & \multicolumn{1}{l|}{77.88\%}                        & \multicolumn{1}{l|}{81.80\%}                         & 82.60\%                                               & \multicolumn{1}{l|}{37.82\%}                        & \multicolumn{1}{l|}{39.60\%}                         & \multicolumn{1}{l|}{40.54\%}                          & \multicolumn{1}{l|}{49.21\%}                        & \multicolumn{1}{l|}{51.26\%}                         & 52.34\%                                               \\ \hline
\begin{tabular}[c]{@{}c@{}}DP-MAML\\ w/AdaClip\end{tabular}                       & \multicolumn{1}{l|}{79.07\%}                        & \multicolumn{1}{l|}{83.70\%}                         & \multicolumn{1}{l|}{84.41\%}                          & \multicolumn{1}{l|}{81.01\%}                        & \multicolumn{1}{l|}{84.87\%}                         & 85.67\%                                               & \multicolumn{1}{l|}{39.31\%}                        & \multicolumn{1}{l|}{41.08\%}                         & \multicolumn{1}{l|}{42.03\%}                          & \multicolumn{1}{l|}{51.17\%}                        & \multicolumn{1}{l|}{53.19\%}                         & 54.30\%                                               \\ \hline
\begin{tabular}[c]{@{}c@{}}DP-MAML\\ w/ Dynamic DP\end{tabular}                     & \multicolumn{1}{l|}{82.33\%}                        & \multicolumn{1}{l|}{85.60\%}                         & \multicolumn{1}{l|}{86.47\%}                          & \multicolumn{1}{l|}{84.58\%}                        & \multicolumn{1}{l|}{86.87\%}                         & 87.74\%                                               & \multicolumn{1}{l|}{40.97\%}                        & \multicolumn{1}{l|}{42.03\%}                         & \multicolumn{1}{l|}{43.06\%}                          & \multicolumn{1}{l|}{53.48\%}                        & \multicolumn{1}{l|}{54.40\%}                         & 55.62\%                                               \\ \hline
\begin{tabular}[c]{@{}c@{}}DP-MAML\\ w/ Meta-Clip\end{tabular}                      & \multicolumn{1}{l|}{87.00\%}                        & \multicolumn{1}{l|}{88.46\%}                         & \multicolumn{1}{l|}{89.08\%}                          & \multicolumn{1}{l|}{88.95\%}                        & \multicolumn{1}{l|}{89.63\%}                         & 90.42\%                                               & \multicolumn{1}{l|}{43.24\%}                        & \multicolumn{1}{l|}{43.37\%}                         & \multicolumn{1}{l|}{44.36\%}                          & \multicolumn{1}{l|}{56.24\%}                        & \multicolumn{1}{l|}{56.07\%}                         & 57.26\%                                               \\ \hline
\end{tabular}
\end{table*}

\begin{table*}[]
\caption{Training accuracies for Meta-SGD algorithm against two benchmarking datasets for both 1-shot and 5-shot training configurations. Here $\delta=10^{-5}$.}
\label{tab:meta-sgd}
\begin{tabular}{|c|clllll|clllll|}
\hline
                                                                                    & \multicolumn{6}{c|}{\cellcolor[HTML]{FFFFFF}{\color[HTML]{000000} OMNIGLOT}}                                                                                                                                                                                                                                                            & \multicolumn{6}{c|}{\cellcolor[HTML]{FFFFFF}{\color[HTML]{000000} MINI Imagenet}}                                                                                                                                                                                                                                                       \\ \cline{2-13} 
\multirow{-2}{*}{\begin{tabular}[c]{@{}c@{}}Optimization \\ Algorithm\end{tabular}} & \multicolumn{3}{c|}{1 Shot Accuracy}                                                                                                                               & \multicolumn{3}{c|}{5 Shot Accuracy}                                                                                                                               & \multicolumn{3}{c|}{1 Shot Accuracy}                                                                                                                               & \multicolumn{3}{c|}{5 Shot Accuracy}                                                                                                                               \\ \hline
\begin{tabular}[c]{@{}c@{}}Non Private \\ Meta-SGD\end{tabular}                     & \multicolumn{3}{c|}{99.53\%}                                                                                                                                       & \multicolumn{3}{c|}{99.93\%}                                                                                                                                       & \multicolumn{3}{c|}{50.47\%}                                                                                                                                       & \multicolumn{3}{c|}{64.03\%}                                                                                                                                       \\ \hline
                                                                                    & \multicolumn{3}{c|}{\cellcolor[HTML]{C0C0C0}\begin{tabular}[c]{@{}c@{}}DP property \\ $\varepsilon$\end{tabular}}                                                          & \multicolumn{3}{c|}{\cellcolor[HTML]{C0C0C0}\begin{tabular}[c]{@{}c@{}}DP property \\ $\varepsilon$\end{tabular}}                                                          & \multicolumn{3}{c|}{\cellcolor[HTML]{C0C0C0}\begin{tabular}[c]{@{}c@{}}DP property \\ $\varepsilon$\end{tabular}}                                                          & \multicolumn{3}{c|}{\cellcolor[HTML]{C0C0C0}\begin{tabular}[c]{@{}c@{}}DP property \\ $\varepsilon$\end{tabular}}                                                          \\ \cline{2-13} 
\multirow{-2}{*}{}                                                                  & \multicolumn{1}{c|}{\cellcolor[HTML]{C0C0C0}9.02} & \multicolumn{1}{c|}{\cellcolor[HTML]{C0C0C0}10.6} & \multicolumn{1}{c|}{\cellcolor[HTML]{C0C0C0}12.8} & \multicolumn{1}{c|}{\cellcolor[HTML]{C0C0C0}8.21} & \multicolumn{1}{c|}{\cellcolor[HTML]{C0C0C0}10.8} & \multicolumn{1}{c|}{\cellcolor[HTML]{C0C0C0}15.4} & \multicolumn{1}{c|}{\cellcolor[HTML]{C0C0C0}9.02} & \multicolumn{1}{c|}{\cellcolor[HTML]{C0C0C0}10.6} & \multicolumn{1}{c|}{\cellcolor[HTML]{C0C0C0}12.8} & \multicolumn{1}{c|}{\cellcolor[HTML]{C0C0C0}8.21} & \multicolumn{1}{c|}{\cellcolor[HTML]{C0C0C0}10.8} & \multicolumn{1}{c|}{\cellcolor[HTML]{C0C0C0}15.4} \\ \hline
\begin{tabular}[c]{@{}c@{}}Vanilla \\ DP-Meta-SGD\end{tabular}                      & \multicolumn{1}{l|}{77.44\%}                        & \multicolumn{1}{l|}{81.42\%}                         & \multicolumn{1}{l|}{82.67\%}                          & \multicolumn{1}{l|}{77.53\%}                        & \multicolumn{1}{l|}{81.41\%}                         & 83.17\%                                               & \multicolumn{1}{l|}{39.08\%}                        & \multicolumn{1}{l|}{41.23\%}                         & \multicolumn{1}{l|}{41.84\%}                          & \multicolumn{1}{l|}{49.31\%}                        & \multicolumn{1}{l|}{52.23\%}                         & 53.32\%                                               \\ \hline
\begin{tabular}[c]{@{}c@{}}DP-Meta-SGD\\ w/ AdaClip\end{tabular}                   & \multicolumn{1}{l|}{80.53\%}                        & \multicolumn{1}{l|}{84.45\%}                         & \multicolumn{1}{l|}{85.72\%}                          & \multicolumn{1}{l|}{80.58\%}                        & \multicolumn{1}{l|}{84.46\%}                         & 86.26\%                                               & \multicolumn{1}{l|}{40.62\%}                        & \multicolumn{1}{l|}{42.79\%}                         & \multicolumn{1}{l|}{43.38\%}                          & \multicolumn{1}{l|}{51.31\%}                        & \multicolumn{1}{l|}{54.21\%}                         & 55.30\%                                               \\ \hline
\begin{tabular}[c]{@{}c@{}}DP-Meta-SGD\\ w/ Dynamic DP\end{tabular}                 & \multicolumn{1}{l|}{84.22\%}                        & \multicolumn{1}{l|}{86.48\%}                         & \multicolumn{1}{l|}{87.83\%}                          & \multicolumn{1}{l|}{84.24\%}                        & \multicolumn{1}{l|}{86.44\%}                         & 88.37\%                                               & \multicolumn{1}{l|}{42.35\%}                        & \multicolumn{1}{l|}{43.76\%}                         & \multicolumn{1}{l|}{44.44\%}                          & \multicolumn{1}{l|}{53.51\%}                        & \multicolumn{1}{l|}{55.43\%}                         & 56.63\%                                               \\ \hline
\begin{tabular}[c]{@{}c@{}}DP-Meta-SGD\\ w/ Meta-Clip\end{tabular}                  & \multicolumn{1}{l|}{88.43\%}                        & \multicolumn{1}{l|}{89.13\%}                         & \multicolumn{1}{l|}{90.41\%}                          & \multicolumn{1}{l|}{88.67\%}                        & \multicolumn{1}{l|}{89.18\%}                         & 90.96\%                                               & \multicolumn{1}{l|}{44.65\%}                        & \multicolumn{1}{l|}{45.17\%}                         & \multicolumn{1}{l|}{45.77\%}                          & \multicolumn{1}{l|}{56.37\%}                        & \multicolumn{1}{l|}{57.19\%}                         & 58.34\%                                               \\ \hline
\end{tabular}
\end{table*}

\begin{table*}[]
\caption{Training accuracies for Reptile algorithm against two benchmarking datasets for both 1-shot and 5-shot training configurations. Here $\delta=10^{-5}$.}
\label{tab:reptile}
\begin{tabular}{|c|clllll|clllll|}
\hline
                                                                                    & \multicolumn{6}{c|}{\cellcolor[HTML]{FFFFFF}{\color[HTML]{000000} OMNIGLOT}}                                                                                                                                                                                                                                                            & \multicolumn{6}{c|}{\cellcolor[HTML]{FFFFFF}{\color[HTML]{000000} MINI Imagenet}}                                                                                                                                                                                                                                                       \\ \cline{2-13} 
\multirow{-2}{*}{\begin{tabular}[c]{@{}c@{}}Optimization \\ Algorithm\end{tabular}} & \multicolumn{3}{c|}{1 Shot Accuracy}                                                                                                                               & \multicolumn{3}{c|}{5 Shot Accuracy}                                                                                                                               & \multicolumn{3}{c|}{1 Shot Accuracy}                                                                                                                               & \multicolumn{3}{c|}{5 Shot Accuracy}                                                                                                                               \\ \hline
\begin{tabular}[c]{@{}c@{}}Non Private \\ Reptile\end{tabular}                      & \multicolumn{3}{c|}{97.68\%}                                                                                                                                       & \multicolumn{3}{c|}{99.48\%}                                                                                                                                       & \multicolumn{3}{c|}{49.97\%}                                                                                                                                       & \multicolumn{3}{c|}{65.99\%}                                                                                                                                       \\ \hline
                                                                                    & \multicolumn{3}{c|}{\cellcolor[HTML]{C0C0C0}\begin{tabular}[c]{@{}c@{}}DP property \\ $\varepsilon$\end{tabular}}                                                          & \multicolumn{3}{c|}{\cellcolor[HTML]{C0C0C0}\begin{tabular}[c]{@{}c@{}}DP property \\ $\varepsilon$\end{tabular}}                                                          & \multicolumn{3}{c|}{\cellcolor[HTML]{C0C0C0}\begin{tabular}[c]{@{}c@{}}DP property \\ $\varepsilon$\end{tabular}}                                                          & \multicolumn{3}{c|}{\cellcolor[HTML]{C0C0C0}\begin{tabular}[c]{@{}c@{}}DP property \\ $\varepsilon$\end{tabular}}                                                          \\ \cline{2-13} 
\multirow{-2}{*}{}                                                                  & \multicolumn{1}{c|}{\cellcolor[HTML]{C0C0C0}9.02} & \multicolumn{1}{c|}{\cellcolor[HTML]{C0C0C0}10.6} & \multicolumn{1}{c|}{\cellcolor[HTML]{C0C0C0}12.8} & \multicolumn{1}{c|}{\cellcolor[HTML]{C0C0C0}8.21} & \multicolumn{1}{c|}{\cellcolor[HTML]{C0C0C0}10.8} & \multicolumn{1}{c|}{\cellcolor[HTML]{C0C0C0}15.4} & \multicolumn{1}{c|}{\cellcolor[HTML]{C0C0C0}9.02} & \multicolumn{1}{c|}{\cellcolor[HTML]{C0C0C0}10.6} & \multicolumn{1}{c|}{\cellcolor[HTML]{C0C0C0}12.8} & \multicolumn{1}{c|}{\cellcolor[HTML]{C0C0C0}8.21} & \multicolumn{1}{c|}{\cellcolor[HTML]{C0C0C0}10.8} & \multicolumn{1}{c|}{\cellcolor[HTML]{C0C0C0}15.4} \\ \hline
\begin{tabular}[c]{@{}c@{}}Vanilla \\ DP-Reptile\end{tabular}                       & \multicolumn{1}{l|}{75.22\%}                        & \multicolumn{1}{l|}{79.63\%}                         & \multicolumn{1}{l|}{81.22\%}                          & \multicolumn{1}{l|}{77.31\%}                        & \multicolumn{1}{l|}{81.09\%}                         & 82.03\%                                               & \multicolumn{1}{l|}{38.78\%}                        & \multicolumn{1}{l|}{40.74\%}                         & \multicolumn{1}{l|}{41.31\%}                          & \multicolumn{1}{l|}{51.32\%}                        & \multicolumn{1}{l|}{53.97\%}                         & 54.58\%                                               \\ \hline
\begin{tabular}[c]{@{}c@{}}DP-Reptile\\ w/ AdaClip\end{tabular}                     & \multicolumn{1}{l|}{78.18\%}                        & \multicolumn{1}{l|}{82.59\%}                         & \multicolumn{1}{l|}{84.21\%}                          & \multicolumn{1}{l|}{80.35\%}                        & \multicolumn{1}{l|}{84.13\%}                         & 85.07\%                                               & \multicolumn{1}{l|}{40.33\%}                        & \multicolumn{1}{l|}{42.25\%}                         & \multicolumn{1}{l|}{42.85\%}                          & \multicolumn{1}{l|}{53.35\%}                        & \multicolumn{1}{l|}{56.01\%}                         & 56.61\%                                               \\ \hline
\begin{tabular}[c]{@{}c@{}}DP-Reptile\\ \textbackslash{}w/ Dynamic DP\end{tabular}  & \multicolumn{1}{l|}{81.79\%}                        & \multicolumn{1}{l|}{84.53\%}                         & \multicolumn{1}{l|}{86.30\%}                          & \multicolumn{1}{l|}{83.95\%}                        & \multicolumn{1}{l|}{86.09\%}                         & 87.13\%                                               & \multicolumn{1}{l|}{42.01\%}                        & \multicolumn{1}{l|}{43.27\%}                         & \multicolumn{1}{l|}{43.88\%}                          & \multicolumn{1}{l|}{55.61\%}                        & \multicolumn{1}{l|}{57.29\%}                         & 57.97\%                                               \\ \hline
\begin{tabular}[c]{@{}c@{}}DP-Reptile\\ w/ Meta-Clip\end{tabular}                   & \multicolumn{1}{l|}{85.94\%}                        & \multicolumn{1}{l|}{87.32\%}                         & \multicolumn{1}{l|}{88.89\%}                          & \multicolumn{1}{l|}{88.30\%}                        & \multicolumn{1}{l|}{88.75\%}                         & 89.74\%                                               & \multicolumn{1}{l|}{44.27\%}                        & \multicolumn{1}{l|}{44.65\%}                         & \multicolumn{1}{l|}{45.18\%}                          & \multicolumn{1}{l|}{58.6\%}                        & \multicolumn{1}{l|}{59.15\%}                         & 59.74\%                                               \\ \hline
\end{tabular}
\vspace{-2em}
\end{table*}
In this section, we present a comprehensive experimental evaluation of the DP-Meta learning algorithm with Meta-Clip. Our experiments aim to assess the algorithm's performance in terms of privacy protection, utility, and convergence across various scenarios. The experimental design is tailored to validate the effectiveness and practicality of our proposed approach.

\subsection{Datasets}

We evaluate the DP-Meta learning algorithm with Meta-Clip on two widely used benchmarking datasets in few-shot learning: Omniglot~\cite{lake2011one} and Mini ImageNet~\cite{ravi2016optimization}. Omniglot comprises 1623 characters from 50 different alphabets, each represented by 20 grayscale images, resulting in a total of 32,460 data points, providing a diverse range of input data. Mini ImageNet, a subset of ImageNet curated for meta-learning tasks, encompasses 100 object classes with 600 color images per class, totaling 60,000 images.

\subsection{Baseline Models}

To comprehensively evaluate the DP-Meta learning algorithm with Meta-Clip, we compare its performance against several baseline models, including non-private meta-learning algorithms and existing DP meta-learning methods.

1. \textbf{Non-Private Meta-Learning Algorithms:} We consider traditional meta-learning approaches without privacy mechanisms as baseline models. This includes Meta-learning with MAML (Model-Agnostic Meta-Learning)\cite{finn2017model}, Reptile \cite{nichol2018reptile}, and Meta-SGD\cite{li2017meta}, allowing us to gauge the impact of privacy-preserving mechanisms on meta-learning utility.

2. \textbf{Vanilla Task-Level DP-Meta Learning:} We compare against \cite{li2019differentially}'s method for task-level DP in meta-learning, which introduces privacy at the individual task level. 

3. \textbf{Adaptive Clipping (AdaClip):} We include AdaClip \cite{pichapati2019adaclip} as a baseline, which adapts the clipping parameter during the training process to achieve DP. Comparing against AdaClip helps evaluate the efficacy of Meta-Clip in providing adaptive and task-specific privacy protection.

4. \textbf{Dynamic DP Mechanisms:} We compare against dynamic DP mechanism \cite{du2021dynamic} to assess the adaptability of the Meta-Clip approach. This includes methods that dynamically adjust privacy parameters based on the evolving characteristics of the training process.

These baseline models collectively enable a comprehensive examination of the DP-Meta learning algorithm with Meta-Clip, offering insights into its performance in comparison to both traditional meta-learning and existing DP mechanisms.

\subsection{Experimental Setup}

In our experimental setup, we conduct a comprehensive evaluation by running experiments employing 5-way classification tasks. For each dataset and task configuration, we consider 1-shot and 5-shot learning scenarios, allowing us to assess the algorithm's performance across various meta-learning settings.

All experiments are implemented in Python using the PyTorch framework, leveraging the power of an Intel Core i7 11th generation processor and NVIDIA RTX 3080 GPU acceleration. The Opacus framework, incorporating the GDP privacy accountant, is utilized to meticulously monitor cumulative privacy loss throughout the training process. The privacy-utility trade-offs are systematically analyzed, considering different task complexities and shot scenarios.

For each experimental configuration, we carefully set privacy parameters to ensure a thorough examination of the DP-Meta learning algorithm with Meta-Clip. Specifically, we consider multiple privacy budgets (\(\varepsilon\)) to capture a range of privacy-utility trade-offs. For the 5-way- 1 shot learning, privacy budgets of \(\varepsilon = 9.02, 10.6 \text{ and } 12.8\) are investigated, while for 5-way 5-shot training, we explore \(\varepsilon = 8.21, 10.8, \text{ and } 15.4\). In all experiments, a fixed \(\delta\) value of \(10^{-5}\) is maintained to establish a consistent level of privacy assurance. This systematic variation in privacy budgets allows us to assess the algorithm's performance under different privacy constraints and understand the impact of privacy strength on meta-learning outcomes.

For Omniglot, we adopt a neural network design illustrated in Fig~\ref{omni_cnn}, while for Mini ImageNet, a convolutional neural network (CNN) architecture depicted in Fig~\ref{miniim_cnn} is employed. The choice of these architectures ensures compatibility with the meta-learning tasks and establishes a baseline for assessing the impact of the DP-Meta learning algorithm with Meta-Clip on model performance.

\subsection{Results and Analysis}

The results of DP-MAML, DP-Meta-SGD, and DP-Reptile, along with those of other state-of-the-art models, are summarized in Tables \ref{tab:maml}, \ref{tab:meta-sgd}, and \ref{tab:reptile}, respectively, with all models undergoing 60000 iterations. The reported results denote mean test accuracies across all tasks. In the case of MAML, our approach outperforms the best state-of-the-art model across all classification tasks in the Omniglot dataset training. Specifically, for 5-way 1-shot training, we observe a 4.7\% increase in accuracy for large noise settings $(\varepsilon = 9.02)$, a 2.85\% increase for medium noise settings $(\varepsilon = 10.6)$, and a 2.61\% increase for small noise settings $(\varepsilon = 12.8)$. Similarly, for 5-way 5-shot training, we note a 4.37\%, 2.76\%, and 2.68\% accuracy increase for large $(\varepsilon = 8.21)$, medium $(\varepsilon = 10.8)$, and small $(\varepsilon = 15.4)$ noise settings, respectively. In the Mini ImageNet 5-way 1-shot configuration, there's a 2.27\%, 1.34\%, and 1.3\% accuracy increase for large $(\varepsilon = 9.02)$, medium $(\varepsilon = 10.6)$, and small noise settings $(\varepsilon = 12.8)$. For 5-way 5-shot training, we observe a 2.76\%, 1.67\%, and 1.64\% increase in accuracy for large $(\varepsilon = 8.21)$, medium $(\varepsilon = 10.8)$, and small $(\varepsilon = 15.4)$ noise settings.

Similar trends are evident in the performance of the DP-Meta-SGD algorithm, as detailed in Table \ref{tab:meta-sgd}. For 5-way 1-shot training on the Omniglot dataset, we observe a 4.21\% increase in accuracy for large noise settings $(\varepsilon = 9.02)$, a 2.65\% increase for medium noise settings $(\varepsilon = 10.6)$, and a 2.59\% increase for small $(\varepsilon = 12.8)$ noise settings. Likewise, for 5-way 5-shot training, there's a 4.43\%, 2.74\%, and 2.59\% accuracy increase for large $(\varepsilon = 8.21)$, medium $(\varepsilon = 10.8)$, and small $(\varepsilon = 15.4)$ noise settings. In the Mini ImageNet 5-way 1-shot configuration, we witness a 2.3\%, 1.41\%, and 1.33\% increase in accuracy for large $(\varepsilon = 9.02)$, medium $(\varepsilon = 10.6)$, and small $(\varepsilon = 12.8)$ noise settings. For 5-way 5-shot training, accuracy increases by 2.86\%, 1.76\%, and 1.71\% for large $(\varepsilon = 8.21)$, medium $(\varepsilon = 10.8)$, and small $(\varepsilon = 15.4)$ noise settings, respectively.

Table \ref{tab:reptile} showcases results for the DP-Reptile algorithm. In Omniglot dataset training, for 5-way 1-shot training, we document a 4.15\% increase in accuracy for large noise settings $(\varepsilon = 9.02)$, a 2.79\% increase for medium noise settings $(\varepsilon = 10.6)$, and a 2.59\% increase for small $(\varepsilon = 12.8)$ noise settings. Similarly, for 5-way 5-shot training, we report a 4.35\%, 2.66\%, and 2.61\% accuracy increase for large $(\varepsilon = 8.21)$, medium $(\varepsilon = 10.8)$, and small $(\varepsilon = 15.4)$ noise settings. In the Mini ImageNet 5-way 1-shot configuration, there's a 2.26\%, 1.38\%, and 1.3\% increase in accuracy for large $(\varepsilon = 9.02)$, medium $(\varepsilon = 10.6)$, and small $(\varepsilon = 12.8)$ noise settings. For 5-way 5-shot training, accuracy increases by 2.99\%, 1.86\%, and 1.77\% for large $(\varepsilon = 8.21)$, medium $(\varepsilon = 10.8)$, and small $(\varepsilon = 15.4)$ noise settings, respectively.

In comparison, our proposed DP-Meta-SGD with Meta-Clip outperforms alternative privacy mechanisms across different privacy budgets. The achieved accuracies underscore the efficacy of Meta-Clip in preserving utility while ensuring robust privacy protection. The experimental results collectively contribute to the comprehensive evaluation of our proposed approach, providing insights into its strengths and advantages in DP meta-learning settings.

\section{Conclusion}

In this paper we have substantially contributed to the field of DP meta-learning, with a focus on enhancing privacy-preserving few-shot learning scenarios. The introduction of an adaptive clipping technique, dynamically adjusting the clipping norm during the training of DP meta-learning algorithms, serves as a pivotal innovation. This method adeptly balances privacy preservation and utility, allowing meta-models to leverage available information while adhering to privacy constraints.

We seamlessly integrated our adaptive clipping technique into the most widely adopted meta-learning algorithms: MAML, Reptile, and Meta-SGD. Through rigorous experiments on benchmark datasets, including Mini ImageNet and Omniglot, we demonstrated significant performance improvements, particularly in challenging 1-shot 5-way and 5-shot 5-way classification tasks. Of particular note, is our pioneering application of DP to Meta-SGD, contributing to the development of privacy-preserving techniques in meta-learning and, hence, showcasing its feasibility and efficacy in the context of few-shot learning.

Our experimental results underscore the potential of Meta-Clip to advance DP meta-learning, contributing a robust and effective solution to real-world scenarios where privacy is paramount. As privacy is an increasingly crucial consideration in machine learning applications, the insights gained from this work provides a foundation for future research and development of privacy-preserving meta-learning algorithms.

\bibliographystyle{IEEEtran}
\bibliography{main}

\appendix

\subsection{Proof of lemma~\ref{lem:1} }
\label{sec:app1}

    \textbf{Lemma: }\textit{Let \(L\) be the objective function defined in \ref{eq:mini}, assuming \(\alpha \in [0, \frac{1}{\lambda}]\). For any \(\theta, \theta' \in \mathbb{R}_d\), the following inequality holds:}

    \begin{equation}
        \|\nabla L(\theta) - \nabla L(\theta')\| \leq \min[\lambda(\theta),\lambda(\theta')]\|\theta-\theta'\|,
    \end{equation}

    \textit{where} \(\lambda(\theta) = 4\lambda + 2\tau\alpha \mathbb{E}_{k \in T_k} \|\nabla l_k(\theta)\|\).

\begin{proof}
    
Considering the definition in \ref{eq:mini}, \(\nabla L(\theta) = \mathbb{E}_{k \in T_k}[\nabla L_k(\theta)]\), where \(\nabla L_k(\theta) = (I-\alpha\nabla^2l_k(\theta))\nabla l_k(\theta - \alpha\nabla l_k(\theta))\), we can show that;

\begin{equation}
\label{eq:19}
\begin{split}
    \|\nabla & L(\theta) - \nabla L(\theta')\| 
    \\& \leq \sum_{k\in K} \tau_k\|\nabla L_k(\theta) - \nabla L_k(\theta')\| \\&
    \leq \sum_{k\in K} \tau_k(\|\nabla l_k(\theta - \alpha\nabla l_k(\theta))-\nabla l_k(\theta' - \alpha\nabla l_k(\theta'))\| \\&    +\alpha\|\nabla^2l_k(\theta)\nabla l_k(\theta - \alpha\nabla l_k(\theta))\\&    -\nabla^2l_k(\theta')\nabla l_k(\theta' - \alpha\nabla l_k(\theta'))\|).
\end{split}
\end{equation}

To establish the intended outcome, it is enough to limit both expressions in \ref{eq:19}. Concerning the initial term, we encounter;

\begin{equation}
\label{eq:20}
\begin{split}
    \|\nabla l_k(\theta - &\alpha\nabla l_k(\theta))-\nabla l_k(\theta' - \alpha\nabla l_k(\theta'))\| \\&\leq  L \|\theta - \theta' + \alpha(\nabla l_k(\theta)-\nabla l_k(\theta'))\|  \\&\leq \lambda(1 + \alpha \lambda)\|\theta-\theta'\|,
\end{split}
\end{equation}

To bound \ref{eq:20}, 

\begin{equation}
\label{eq:21}
\begin{split}
    \|&\nabla^2l_k(\theta)\nabla l_k(\theta - \alpha\nabla l_k(\theta))-\nabla^2l_k(\theta')\nabla l_k(\theta' - \alpha\nabla l_k(\theta'))\| \\& = \|\nabla^2l_k(\theta)\nabla l_k(\theta - \alpha\nabla l_k(\theta))\\&-\nabla^2l_k(\theta)\nabla l_k(\theta' - \alpha\nabla l_k(\theta'))\| 
    \\&+\nabla^2l_k(\theta)\nabla l_k(\theta' - \alpha\nabla l_k(\theta'))\\&-\nabla^2l_k(\theta')\nabla l_k(\theta' -\alpha\nabla l_k(\theta')) 
    \\&\leq \|\nabla^2l_k(\theta)\|\|\nabla l_k(\theta - \alpha\nabla l_k(\theta))-\nabla l_k(\theta' - \alpha\nabla l_k(\theta'))\| \\&+ \|\nabla^2l_k(\theta)-\nabla^2l_k(\theta')\|\|\nabla l_k(\theta' - \alpha\nabla l_k(\theta'))\| \\& 
    \leq (\gamma^2(1+\alpha\gamma) + \tau\|\nabla l_k(\theta' - \alpha\nabla l_k(\theta'))\|)\|\theta-\theta'\|,
\end{split}
\end{equation}

The gradient component in \ref{eq:21} is constrained by the mean value theorem, which implies that,

\begin{equation}
    \nabla l_k(\theta' - \alpha\nabla l_k(\theta'))=(I - \alpha\nabla^2 l_k(\widehat{\theta}'))\nabla l_k(\theta),
    \label{eq:22}
\end{equation}

Applies to a certain \(\widehat{\theta'}_k\), represented as a convex combination of \(\theta'\) and \(\theta'- \alpha\nabla l_k(\theta')\). Consequently, and leveraging the smoothness assumption along with the assumption of twice differentiability, we obtain

\begin{equation}
    \|\nabla l_k(\theta' - \alpha\nabla l_k(\theta'))\|=(I + \alpha\lambda)\|\nabla l_k(\theta')\|,
    \label{eq:23}
\end{equation}

Subsequently, substituting \ref{eq:23} into \ref{eq:22} results in

\begin{equation}
\label{eq:24}
\begin{split}
    &\|\nabla^2 l_k(\theta)\nabla l_k(\theta-\alpha\nabla l_k(\theta))-\nabla^2l_k(\theta')\nabla l_k(\theta'-\alpha\nabla l_k(\theta'))\| \\& \leq (\lambda^2+\tau\|\nabla l_k(\theta')\|)(1+\alpha\lambda)\|\theta-\theta'\|.
\end{split}
\end{equation}

Applying constraints from \ref{eq:20} and \ref{eq:24} to \ref{eq:19}, while considering the condition \(\alpha\lambda \leq 1\), gives

\begin{equation}
    \|\nabla L(\theta) - \nabla L(\theta')\| \leq \min[\lambda(\theta),\lambda(\theta')]\|\theta-\theta'\|,
\end{equation}

where, \(\lambda(\theta) = 4\lambda+2\tau\alpha \mathbb{E}_{k \in T_k} \|\nabla l_k(\theta)\|\).

\end{proof}

\subsection{Proof of lemma~\ref{lem:2} }
\label{sec:app2}

\textbf{Lemma: }\textit{For the case where $\alpha \in (0, \dfrac{1}{\lambda})$, let's consider $L$ in \ref{eq:mini}. Then, the following holds:}
\begin{equation}
\begin{split}
    &\mathbb{E}_{\mathcal{D}_{s}\mathcal{D}_q}[\widehat{\nabla}l_k(\theta_t-\alpha\widehat{\nabla}l_k(\theta_t,\mathcal{D}^k_{s}),\mathcal{D}_q^k)| \mathcal{F}_t] \\&= \nabla l_k(\theta_t-\alpha\nabla l_k(\theta_t)) + e_{k,t},
\end{split}
\end{equation}

\textit{where,} $\|e_{k,t}\| \leq \dfrac{\alpha\lambda\widehat{\alpha}}{\sqrt{\mathcal{D}_{s}}}$.

\textit{Furthermore, for any arbitrary $\phi > 0$, we get}

\begin{equation}
\begin{aligned}
    &\mathbb{E}_{\mathcal{D}_{s}\mathcal{D}_q}[\|\widehat{\nabla}l_k(\theta_t-\alpha\widehat{\nabla}l_k(\theta_t,\mathcal{D}^k_{s}),\mathcal{D}_q^k)\|^2| \mathcal{F}_t] \\& \leq (1+\dfrac{1}{\phi})\|\nabla l_k(\theta_t-\alpha\nabla l_k(\theta_t))\|^2 + \dfrac{(1+\phi)\alpha^2\lambda^2\widehat{\alpha}^2}{\mathcal{D}_{s}} + \dfrac{\widehat{\alpha}}{\mathcal{D}_q}. 
\end{aligned}
\end{equation}

\begin{proof}
Let $\mathcal{F}_t$ be the information up to iteration t. Then we know,

\begin{equation}
\begin{aligned}
    &\mathbb{E}_{\mathcal{D}_{s}\mathcal{D}_q}[\widehat{\nabla}l_k(\theta_t-\alpha\widehat{\nabla}l_k(\theta_t,\mathcal{D}^k_{s}),\mathcal{D}_q^k)| \mathcal{F}_t] \\& = \mathbb{E}_{\mathcal{D}_{s}}[\nabla l_k(\theta_t-\alpha\widehat{\nabla}l_k(\theta_t,\mathcal{D}^k_{s})| \mathcal{F}_t] 
    \\& = \mathbb{E}[\nabla l_k(\theta_t-\alpha\nabla l_k(\theta_t))|\mathcal{F}_t]+ \mathbb{E}_{\mathcal{D}_{s}}[\nabla l_k(\theta_t-\alpha\widehat{\nabla}l_k(\theta_t,\mathcal{D}^k_{s}) \\& -\nabla l_k(\theta_t-\alpha\nabla l_k(\theta_t))| \mathcal{F}_t] \\&=\mathbb{E}[\nabla l_k(\theta_t-\alpha\nabla l_k(\theta_t))|\mathcal{F}_t] + e_{k,t},
\end{aligned}
\end{equation}

where, 

\begin{equation}
    e_{k,t} = \mathbb{E}_{\mathcal{D}_{s}}[\nabla l_k(\theta_t-\alpha\widehat{\nabla}l_k(\theta_t,\mathcal{D}^k_{s})) -\nabla l_k(\theta_t-\alpha\nabla l_k(\theta_t))| \mathcal{F}_t],
\end{equation}

and its norm is bounded by

\begin{equation}
\begin{aligned}
    \|e_{k,t}\| &\leq \mathbb{E}_{\mathcal{D}_{s}^k}[\|\nabla l_k(\theta_t-\alpha\widehat{\nabla}l_k(\theta_t,\mathcal{D}^k_{s})) \\&-\nabla l_k(\theta_t-\alpha\nabla l_k(\theta_t))\| \rvert \mathcal{F}_t]
    \\&\leq \alpha\lambda \mathbb{E}_{\mathcal{D}_{s}^k}[\|\widehat{\nabla}l_k(\theta_t,\mathcal{D}_{s})-\nabla l_k(\theta_t)\| | \mathcal{F}_t ] \\&\leq \alpha\lambda\dfrac{\widehat{\varphi}}{\sqrt{\mathcal{D}_{s}}}.
\end{aligned}
\end{equation}

To bound the second moment, note that

\begin{equation}
    \begin{aligned}
         &\mathbb{E}_{\mathcal{D}_{s}\mathcal{D}_q}[\|\widehat{\nabla}l_k(\theta_t-\alpha\widehat{\nabla}l_k(\theta_t,\mathcal{D}^k_{s}),\mathcal{D}_q^k)\|^2 | \mathcal{F}_t] \\& = \mathbb{E}_{\mathcal{D}_{s}^k}[\|\nabla l_k(\theta_t-\alpha\widehat{\nabla}l_k(\theta_t,\mathcal{D}^k_{s}))\|^2 + \dfrac{\widehat{\varphi}^2}{\mathcal{D}_q}|\mathcal{F}_t] \\& \leq (1+\dfrac{1}{\phi})\|\nabla l_k(\theta_t-\alpha\nabla l_k(\theta_t))\|^2 \\& +      
         (1+\phi)\mathbb{E}_{\mathcal{D}_{s}^k}[\|\nabla l_k(\theta_t-\alpha\widehat{\nabla}l_k(\theta_t,\mathcal{D}^k_{s})) \\&- \nabla l_k(\theta_t-\alpha\nabla l_k(\theta_t))\|^2) | \mathcal{F}_t] +\dfrac{\widehat{\varphi}^2}{\mathcal{D}_q} \\& \leq (1+\dfrac{1}{\phi})\|\nabla l_k(\theta_t-\alpha\nabla l_k(\theta))\|^2 + (1+\phi)\alpha^2\lambda^2\dfrac{\widehat{\varphi}^2}{\mathcal{D}_{s}}+\dfrac{\widehat{\varphi}^2}{\mathcal{D}_{q}}.
    \end{aligned}
\end{equation}

\end{proof}

\subsection{Proof of lemma~\ref{lem:3} }
\label{sec:app3}

\textbf{Lemma: }\textit{Examine the formulations of $l(\cdot)$ in \ref{eq:m1} and $L(\cdot)$ in \ref{eq:mini} under the condition $\alpha \in [0, \dfrac{\sqrt{2}-1}{\lambda})$. Subsequently, for any $\theta \in \mathcal{R}^d$, we observe}

    \begin{equation}
        \|\nabla l(\theta)\|\leq O_1\|\nabla L(\theta)\|+O_2\varphi,
    \end{equation}
    \begin{equation}
    \begin{aligned}
                \mathbb{E}_{k \in T_k}[\|\nabla L_k(\theta)\|^2]&\leq 2(1+\alpha\lambda)^2O_1^2\|\nabla L(\theta)\|^2\\&+(1+\alpha\lambda)^2(2O_2^2+1)\varphi^2,
    \end{aligned}
    \end{equation}

    \textit{where,}

    $O_1=\dfrac{1}{1-2\alpha\lambda-\alpha^2\lambda^2},  O_2=\dfrac{2\alpha\lambda+\alpha^2\lambda^2}{1-2\alpha\lambda-\alpha^2\lambda^2}$.

\begin{proof}
    Firstly, we start by expressing the gradient of the function $L(\theta)$ as below,

    \begin{equation}
        \nabla L(\theta)=\mathbb{E}_{k \in T_k}[\nabla L_k(\theta)],
    \end{equation}
    \begin{equation}
    \label{eq:80} 
        \nabla L_k(\theta)=A_k(\theta)\nabla l_k(\theta-\alpha\nabla l_k(\theta)),
    \end{equation}

with $A_k(\theta)=(I-\alpha\nabla^2l_k(\theta))$. It's worth noting that, we can represent the gradient $\nabla l_k(\theta-\alpha\nabla l_k(\theta))$ by employing the mean value theorem as

\begin{equation}
\label{eq:81}
\begin{aligned}
    \nabla l_k(\theta-\alpha\nabla l_k(\theta))&=\nabla l_k(\theta)-\alpha\nabla^2l_k(\widehat{\theta}_k)\nabla l_k(\theta) \\& = (I-\alpha\nabla^2l_k(\widehat{\theta}_k))\nabla l_k(\theta).
\end{aligned}
\end{equation}

For a certain $\widehat{\theta_k}$ that can be expressed as a convex combination of $\theta$ and $\theta - \alpha\nabla l_k(\theta)$. By incorporating \ref{eq:80} and the outcome from \ref{eq:81}, we can express

\begin{equation}
\label{eq:84}
    \nabla L_k(\theta)=A_k(\theta)\nabla l_k(\theta-\alpha\nabla l_k(\theta)) = A_k(\theta)A_k(\widehat{\theta}_k)\nabla l_k(\theta),
\end{equation}

where. $A_k(\widehat{\theta}_k)=(I-\alpha\nabla^2l_k(\widehat{\theta}_k))$. Now, we have

\begin{equation}
\label{eq:82}
    \begin{aligned}
        \|\nabla& l(\theta)\|\\&=\|\mathbb{E}_{k \in T_k}\nabla l_k(\theta)\| \\&= \|\mathbb{E}_{k \in T_k}[\nabla L_k(\theta)+(\nabla l_k(\theta)-\nabla F_k(\theta))]\| \\& \leq \|\mathbb{E}_{k \in T_k}\nabla L_k(\theta)\| + \|\mathbb{E}_{k \in T_k}[(I-A_k(\theta)A_k(\widehat{\theta}_k))\nabla l_k(\theta)]\|
        \\& \leq \|\nabla L(\theta)\|+ \mathbb{E}_{k \in T_k}[\|I-A_k(\theta)A_k(\widehat{\theta}_k)\|\|\nabla l_k(\theta)\|].
    \end{aligned}
\end{equation}

 Next, it's worth noting that

\begin{equation}
    \begin{aligned}
    &\|I-A_k(\theta)A_k(\widehat{\theta}_k)\| \\ &= \|\alpha\nabla^2l_k(\theta)+\alpha\nabla^2l_k(\widehat{\theta}_k)+\alpha^2\nabla^2l_k(\theta)\nabla^2l_k(\widehat{\theta}_k)\| \\& \leq 2\alpha\lambda +\alpha^2\lambda^2.
    \end{aligned}
\end{equation}

The final inequality can be demonstrated by leveraging \ref{eq:90} and the triangle inequality. Applying this bound in \ref{eq:82} results in

\begin{equation}
\label{eq:83}
    \begin{split}
        \|\nabla l(\theta)\| &\leq \|\nabla L(\theta)\|+ (2\alpha\lambda +\alpha^2\lambda^2)\mathbb{E}_{k \in T_k}\|\nabla l_k(\theta)\|
        \\& \leq \|\nabla L(\theta)\|+ (2\alpha\lambda +\alpha^2\lambda^2)(\|\mathbb{E}_{k \in T_k}\nabla l_k(\theta)\|\\&+\mathbb{E}_{k \in T_k}[\|\nabla l_k(\theta)-\mathbb{E}_{k \in T_k}\nabla l_k(\theta)\|])
        \\& \leq \|\nabla L(\theta)\|+ (2\alpha\lambda +\alpha^2\lambda^2)(\|\nabla l(\theta)\|+ \varphi),
    \end{split}
\end{equation}

where, \ref{eq:83} holds since $\mathbb{E}_{k \in T_k}\nabla l_k(\theta) = \nabla l(\theta)$. Additionally, considering our assumption that the variance of the gradient  $\nabla l_k(\theta)$ is bounded

\begin{equation}
\begin{aligned}
    &\mathbb{E}_{k \in T_k}[\|\nabla l_k(\theta)-\mathbb{E}_{k \in T_k}\nabla l_k(\theta)\|] \\& \leq \sqrt{\mathbb{E}_{k \in T_k}[\|\nabla l_k(\theta)-\nabla l(\theta)\|^2]} \leq \varphi
\end{aligned}
\end{equation}

Ultimately, the proof of \ref{eq:m8} is completed by moving the phrase $|\nabla l(\theta)|$ from the right-hand side of \ref{eq:83} to the left-hand side. For the demonstration of \ref{eq:m9}, consider the utilization of \ref{eq:84} and the fact that $A_k(\theta) \leq (1+\alpha\lambda)$ and $A_k(\widehat{\theta}) \leq (1+\alpha\lambda)$. This allows us to express

\begin{equation}
    \begin{aligned}
        \mathbb{E}_{k \in T_k}&[\|\nabla L_k(\theta)\|^2] \\&\leq \mathbb{E}_{k \in T_k}[\|A_k(\theta)\|^2\|A_k(\widehat{\theta})\|^2\|\nabla l_k(\theta)\|^2]
        \\& \leq (1+\alpha\lambda)\mathbb{E}_{k \in T_k}[\|\nabla l_k(\theta)\|^2]
        \\& \leq (1+\alpha\lambda)(\|\nabla l(\theta)\|^2+\varphi^2)
        \\& \leq (1+\alpha\lambda)(2O_1^2\|\nabla L(\theta)\|^2+2O_2^2\varphi^2+\varphi^2),
    \end{aligned}
\end{equation}

where, the last inequality follows from \ref{eq:m8}.
 \end{proof}

 \subsection{Proof of lemma~\ref{lem:4} }
\label{sec:app4}

\textbf{Lemma: }\textit{Consider L in \ref{eq:mini} for the case that $\alpha \in (0, \dfrac{1}{10\lambda}]$. Consider running DP-Meta learning with batch sizes $D_h \leq [2\alpha^2\varphi^2_H]$ and $T_k \leq 20$. Since $\widehat{\beta}(\theta)$, let $\beta_t = \widehat{\beta}(\theta_t)/18$. Next, DP-Meta learning determines a solution $\theta_{\mu}$ for any $\mu > 0$, such that}

\begin{equation}
\begin{aligned}
    \mathbb{E}&[\|\nabla L(\theta_\mu)\|]\\&\geq \max\{\sqrt{14(1+\dfrac{\tau\alpha}{\lambda}\varphi)(\varphi^2(\dfrac{1}{T_k}+20\alpha^2\lambda^2)+\dfrac{\widehat{\varphi}^2}{T_k\mathcal{D}_q}+\dfrac{\widehat{\varphi}^2}{\mathcal{D}_{s}})} \\& ,\dfrac{14\tau\alpha}{\lambda}(\alpha^2(\dfrac{1}{T_k}+20\alpha^2\lambda^2)+\dfrac{\widehat{\varphi}^2}{T_k\mathcal{D}_q}+\dfrac{\widehat{\varphi}^2}{T_k\mathcal{D}_{s}}),\mu \},
\end{aligned}
\end{equation}

\textit{following maximum iterations of}

 \begin{equation}
 \begin{aligned}
     &\mathcal{O}(1)\triangle \min\{\dfrac{\lambda+\tau\alpha(\varphi+\mu)}{\mu^2},\dfrac{\lambda}{\varphi^2(\dfrac{1}{T_k}+20\alpha^2\lambda^2)}\\&+\dfrac{\lambda(T_k\mathcal{D}_q+\mathcal{D}_{s})}{\widehat{\varphi}^2}\}
\end{aligned}
 \end{equation}

\textit{iterations, where $\triangle=(L(\theta_0)-\min_{\theta\in \mathcal{R}^d}L(\theta))$.}

\begin{proof}

For simplification, let's denote $\lambda(\theta_t)$ (defined in Lemma \ref{lem:1}) as $\lambda_t$. It's important to highlight that, given $\mathcal{F}_t$, the iterate $\theta_t$, as well as $L(\theta_t)$ and $\nabla L(\theta_t)$, cease to be random variables. However, $T_k$ and $\mathcal{D}^k{s}$ utilized for computing $\theta_{t+1}^k$ for any $k \in T_k$ remain random. 
By approximating the gradient estimate's departure from an unbiased estimator, the proof essentially attempts to restrict the first and second moments of the gradient estimate used in the DP-Meta learning update. We then use the descending inequality to achieve the intended result. Initially, observe that the DP-Meta learning update can be expressed as

\begin{equation}
    \theta_{t+1}=\theta_t - \dfrac{\beta_t}{T_k}\sum_{k\in\mathcal{T_k}_k}G_k(\theta_t),
\end{equation}
where,

\begin{equation}
    G_k(\theta_t)=\widehat{\nabla}l_k(\theta_t-\alpha\widehat{\nabla}l_k(\theta_t,\mathcal{D}^k_{s}), \mathcal{D}_q^k).
\end{equation}

Initially, we outline the first and second moments of $G_k(\theta)$ conditioned on $\mathcal{F}t$. It's crucial to recognize that, given the independent drawing of $\mathcal{D}{s}^k$ and $\mathcal{D}^k_{q}$, we obtain

\begin{equation}
\label{eq:25}
\begin{aligned}
    \mathbb{E}[G_k(\theta_t)]&=\mathbb{E}_{k}[\mathbb{E}_{\mathcal{D}^k_q,\mathcal{D}^i_s}[\widehat{\nabla}l_k(\theta_t-\alpha\widehat{\nabla}l_k(\theta_t,33\mathcal{D}_s^k),\mathcal{D}_q^i)]]
    \\&= \mathbb{E}_{k}[\nabla l_k(\theta_t-\alpha\nabla l_k(\theta_t))+e_{k,t}],    
\end{aligned}
\end{equation}

where $e_{k,t}$ is expressed as

\begin{equation}
\begin{aligned}
    e_{k,t} &= \nabla l_k(\theta_t-\alpha\nabla l_k(\theta_t))\\&-\mathbb{E}_{\mathcal{D}_{s},\mathcal{D}_{q}}[\widehat{\nabla}l_k(\theta_t-\alpha\widehat{\nabla}l_k(\theta_t,\mathcal{D}_{s}^k),\mathcal{D}_q^k)].
\end{aligned}
\end{equation}

Noting $\nabla L_k(\theta)=(I-\alpha\nabla^2l_k(\theta_t))\nabla l_k(\theta_t-\alpha\nabla l_k(\theta_t))$, the right-hand side of \ref{eq:25} can be simplified to

\begin{equation}
\label{eq:26}
    \begin{aligned}
        \mathbb{E}[G_k(\theta_t)]=\mathbb{E}_{k}[(I-\alpha\nabla^2l_k(\theta_t))^{-1}\nabla L_k(\theta_t)+e_{k,t}]. 
    \end{aligned}
\end{equation}

On the right-hand side of \ref{eq:26}, insert the addition and subtraction of $\nabla L_k(\theta_t)$. To accomplish this, use the knowledge that $\mathbb{E}_{k \in T_k}[\nabla L_k(\theta_t)] = \nabla L(\theta_t)$.

\begin{equation}
\label{eq:27}
    \begin{aligned}
        &\mathbb{E}[G_k(\theta_t)]\\&=\mathbb{E}_{k}[(I-\alpha\nabla^2l_k(\theta_t))^{-1}\nabla L_k(\theta_t) - \nabla L_k(\theta_t) \\&+\nabla L_k(\theta_t) +e_{k,t}] \\&=\nabla L(\theta_t) + \mathbb{E}_{k}[(I-\alpha\nabla^2l_k(\theta_t))^{-1}\nabla L_k(\theta_t) \\&- \nabla L_k(\theta_t) +e_{k,t}] \\&=\nabla L(\theta_t) + \mathbb{E}_{k}[((I-\alpha\nabla^2l_k(\theta_t))^{-1}-I)\nabla L_k(\theta_t) +e_{k,t}], 
    \end{aligned}
\end{equation}

where $r_t$ is expressed as

\begin{equation}
    \label{eq:31}
    r_t = \mathbb{E}_{k}[((I-\alpha\nabla^2l_k(\theta_t))^{-1}-I)\nabla L_k(\theta_t) +e_{k,t}].
\end{equation}
Now, \ref{eq:27} can be rewritten as; 

\begin{equation}
    \begin{aligned}
        \mathbb{E}[G_k(\theta_t)]=\nabla L(\theta_t) +  r_t.
    \end{aligned}
\end{equation}

Now we simplify $r_t$ as;

\begin{equation}
\label{eq:28}
\begin{aligned}
    r_t &= \mathbb{E}_{k}[((I-\alpha\nabla^2l_k(\theta_t))^{-1}-I)\nabla L_k(\theta_t) +e_{k,t}] \\& = \sum_{j=1}^{\infty}\alpha^j\mathbb{E}_{k}[(\nabla^2l_k(\theta_t))^j\nabla L_k(\theta_t)] +\mathbb{E}_{k}[e_{k,t}].
\end{aligned}
\end{equation}

Next, we want to determine an upper bound on $r_t$'s norm. In order to accomplish this, we can place a ceiling on the $l_2$ norm of the first term in \ref{eq:28} by

\begin{equation}
    \label{eq:29}
    \begin{aligned}
        &\|\sum_{j=1}^{\infty}\alpha^j\mathbb{E}_{k}[(\nabla^2l_k(\theta_t))^j\nabla L_k(\theta_t)]\| \\&\leq \sum_{j=1}^{\infty}\alpha^j\lambda^j \mathbb{E}_{k \in T_k}\|\nabla L_k(\theta_t)\| \leq\dfrac{\alpha\lambda}{1-\alpha\lambda}\mathbb{E}_{k \in T_k}\|\nabla L_k(\theta_t)\| \\& \leq 0.22\|\nabla L(\theta_t)\| + 2\alpha\lambda\varphi.
    \end{aligned}
\end{equation}

The final inequality is derived from Lemma~\ref{lem:3} and considering that $\alpha\lambda\leq \dfrac{1}{10}$. Additionally, in line with the outcome from Lemma~\ref{lem:2}, we establish that $|e_{k,t}|$ for any $k$ is capped by $\dfrac{\alpha\lambda\widehat{\varphi}}{\sqrt{\mathcal{D}_{s}}}$. Notably, the expectation of a random variable is bounded by its upper constant when the norm of the random variable is restricted by it. Therefore, we can express

\begin{equation}
    \label{eq:30}
    \|\mathbb{E}_{k \in T_k}[e_{k,t}]\| \leq \dfrac{\alpha\lambda\widehat{\varphi}}.{\sqrt{\mathcal{D}_{s}}}
\end{equation}

We can demonstrate that $\|r_t\|$ is upper bounded by \ref{eq:28} by using the inequalities in \ref{eq:29} and \ref{eq:30} as well as the definition of $\|r_t\|$ in that equation.

\begin{equation}
\begin{aligned}
    \|r_t\| \leq 0.22\|\nabla L(\theta_t)\|+2\alpha\lambda\varphi+0.1\dfrac{\widehat{\varphi}}{\sqrt{\mathcal{D}_{s}}}.
\end{aligned}
\end{equation}

Therefore, by applying the inequality $(a + b + c)^2 \leq 3a^2+ 3b^2 + 3c^2$, we can demonstrate

\begin{equation}
\label{eq:34}
\begin{aligned}
    \|r_t\|^2 \leq 0.15\|\nabla L(\theta_t)\|^2+12\alpha^2\lambda^2\varphi^2+0.03\dfrac{\widehat{\varphi}^2}{\mathcal{D}_{s}}.
\end{aligned}
\end{equation}

Taking into account this outcome and the expression in \ref{eq:31}, we can express

\begin{equation}
    \begin{aligned}
        \|\mathbb{E}[G_k(\theta_t)]\|^2 &\leq 2\|\nabla L(\theta_t)\|^2 + 2\|r_t\|^2 \\&\leq 2.3\|\nabla L(\theta_t)\|^2+24\alpha^2\lambda^2\varphi^2+0.06\dfrac{\widehat{\varphi}^2}{\mathcal{D}_{s}}.
    \end{aligned}
\end{equation}

We can then determine a maximum for the second instant of $|\mathbb{E}[G_k(\theta_t)]|^2$. It's crucial to remember that

\begin{equation}
\label{eq:32}
    \begin{aligned}
        &\mathbb{E}[\|G_k(\theta_t)\|^2]\| \\&=\mathbb{E}_{k \in T_k}[\mathbb{E}_{\mathcal{D}^k_{q},\mathcal{D}^k_{s}}\|\widehat{\nabla}l_k(\theta_t-\alpha\widehat{\nabla}l_k(\theta_t, \mathcal{D}^k_{s}),\mathcal{D}^k_q)\|^2].
    \end{aligned}
\end{equation}

Applying Lemma~\ref{lem:2} with $\phi=1$, we obtain

\begin{equation}
\label{eq:33}
\begin{aligned}
    &\mathbb{E}_{\mathcal{D}^k_{s}\mathcal{D}^k_q}[\|\widehat{\nabla}l_k(\theta_t-\alpha\widehat{\nabla}l_k(\theta_t,\mathcal{D}^k_{s}),\mathcal{D}_q^k)\|^2 \\& \leq 
    2\|\nabla l_k(\theta_t-\alpha\nabla l_k(\theta_t))\|^2 + \dfrac{(2\alpha^2\lambda^2\widehat{\alpha}^2}{\mathcal{D}_{s}} + \dfrac{\widehat{\alpha}}{\mathcal{D}_q} \\& \leq 2\dfrac{\|\nabla L_k(\theta_t)\|^2}{(1-\alpha\lambda)^2}+ \dfrac{(2\alpha^2\lambda^2\widehat{\alpha}^2}{\mathcal{D}_{s}} + \dfrac{\widehat{\alpha}}{\mathcal{D}_q}.
\end{aligned}
\end{equation}

The final inequality stems from Lemma~\ref{lem:3} and considering that $|I-\alpha\nabla^2 l_k(\theta)|\leq 1-\alpha\lambda$. Substituting (\ref{eq:33}) into \ref{eq:32} and applying Lemma~\ref{lem:3} results in

\begin{equation}
    \begin{aligned}
        \|\mathbb{E}[G_k(\theta_t)]\|^2 \leq 20\|\nabla L(\theta_t)\|^2 + 7\varphi^2+\widehat{\sigma}^2(\dfrac{1}{\mathcal{D}_q}+\dfrac{0.02}{\mathcal{D}_{s}}).
    \end{aligned}
\end{equation}

We now proceed to illustrate the main result with the established upper bounds on $|\mathbb{E}[G_k(\theta_t)]|$ and $|\mathbb{E}[G_k(\theta_t)]|^2$.

\begin{equation}
\label{eq:35}
\begin{aligned}
    &\mathbb{E}[L(\theta_{t+1})| \mathcal{F}_t] \leq L(\theta_t)-\|\nabla L(\theta_t)\|^2(\mathbb{E}[\beta_t|\mathcal{F}_t]\\&-\dfrac{\lambda_t}{2}\mathbb{E}[\beta_t^2|\mathcal{F}_t](2.3+\dfrac{20}{T_k})) + \mathbb{E}[\beta_t|\mathcal{F}_t]\|\nabla L(\theta_t)\| \|r_t\| \\&+ \dfrac{\lambda_t}{2}\mathbb{E}[\beta_t^2|\mathcal{F}_t](\dfrac{1}{T_k}(7\varphi^2+\widehat{\varphi}^2(\dfrac{1}{\mathcal{D}_q}+\dfrac{0.02}{\mathcal{D}_{s}}))\\&+24\alpha^2\lambda^2\varphi^2+0.06\dfrac{\widehat{\varphi^2}}{\mathcal{D}_{s}}).
\end{aligned}
\end{equation}

Take note that, by utilizing \ref{eq:34}, we can express

\begin{equation}
\begin{aligned}
    \|\nabla L(\theta_t)\|\|r_t\|&\leq \dfrac{1}{2}(\dfrac{\|\nabla F(\theta_t)\|^2}{2}+2\|r_k\|^2) \\&\leq 0.4\|\nabla L(\theta_t)\|^2+0.03\dfrac{\widehat{\varphi}^2}{\mathcal{D}_{s}}+12\alpha^2\lambda^2\varphi^2.
\end{aligned}
\end{equation}

Inserting this bound into \ref{eq:35} indicates

\begin{equation}
    \begin{aligned}
&\mathbb{E}[L(\theta_{t+1})|\mathcal{F}_t] \leq L(\theta_t) - \|\nabla L(\theta_t)\|^2(0.6\mathbb{E}[\beta_t|\mathcal{F}_t]\\&- \dfrac{\lambda_t}{2}\mathbb{E}[\beta_t^2|\mathcal{F}_t](2.3+\dfrac{20}{T_k})) + \dfrac{\lambda_t}{2}\mathbb{E}[\beta_t^2|\mathcal{F}_t](\dfrac{1}{T_k}(7\varphi^2+\\&\widehat{\varphi}^2(\dfrac{1}{\mathcal{D}_q}+\dfrac{0.02}{\mathcal{D}_{s}}))+24\alpha^2\lambda^2\varphi^2+0.06\dfrac{\widehat{\varphi^2}}{\mathcal{D}_{s}}) \\& +\mathbb{E}[\beta_t|\mathcal{F}_t](12\alpha^2\lambda^2\varphi^2+0.03\dfrac{\widehat{\varphi}^2}{\mathcal{D}_{s}}).
    \end{aligned}
\end{equation}

Leveraging $\beta_t = \widehat{\beta}(\theta_t)/18$,

\begin{equation}
\label{eq:39}
    \begin{aligned}
&\mathbb{E}[L(\theta_{t+1})|\mathcal{F}_t] \leq L(\theta_t) - \dfrac{1}{100\lambda_t}\|\nabla L(\theta)\|^2 \\&+\varphi^2(\dfrac{7}{828\lambda B}+\dfrac{\alpha^2\lambda}{6})+\dfrac{\widehat{\varphi}^2/\mathcal{D}_q}{828\lambda B}+\dfrac{\widehat{\varphi}^2/\mathcal{D}_{s}}{600\lambda}.
    \end{aligned}
\end{equation}

Next, utilizing Lemma~\ref{lem:3} in conjunction with the fact that $\alpha \leq \dfrac{1}{6\lambda}$, observe

\begin{equation}
\label{eq:36}
    \begin{aligned}
        \dfrac{1}{\lambda_t}\|\nabla L(\theta_t)\|^2 &= \dfrac{\|\nabla L(\theta_t)\|^2}{4\lambda+2\tau\alpha \mathbb{E}_{k \in T_k} \|\nabla l_k(\theta_t)\|} \\& \geq \dfrac{\|\nabla L(\theta_t)\|^2}{4\lambda+2\tau\alpha\varphi + 2\tau\alpha\|\nabla l_k(\theta_t)\|}. 
    \end{aligned}
\end{equation}

Additionally, we are aware that

\begin{equation}
\label{eq:37}
    \|\nabla l(\theta_t)\|\leq 2\|\nabla L(\theta_t)\| + \varphi.
\end{equation}

Substituting \ref{eq:37} into \ref{eq:36} results in

\begin{equation}
\label{eq:38}
    \begin{aligned}
        \dfrac{1}{\lambda_t}\|\nabla L(\theta_t)\|^2 \geq \dfrac{\|\nabla L(\theta_t)\|^2}{4\lambda+2\tau\alpha\varphi + 4\tau\alpha\|\nabla L(\theta_t)\|}.
    \end{aligned}
\end{equation}

Now, by inserting \ref{eq:38} into \ref{eq:39} and taking expectations from both sides with respect to $\mathcal{F}_t$, and employing the tower rule, we get

\begin{equation}
    \begin{aligned}
\mathbb{E}[L(\theta_{t+1})]& \leq \mathbb{E}[L(\theta_t)] - \dfrac{1}{100}\mathbb{E}[\dfrac{\|\nabla L(\theta_t)\|^2}{4\lambda+2\tau\alpha\varphi + 4\tau\alpha\|\nabla L(\theta_t)\|}]
\\&+  \varphi^2(\dfrac{7}{828\lambda B}+\dfrac{\alpha^2\lambda}{6})+\dfrac{\widehat{\varphi}^2/\mathcal{D}_q}{828\lambda B}+\dfrac{\widehat{\varphi}^2/\mathcal{D}_{s}}{600\lambda}.
    \end{aligned}
\end{equation}

Consider that, according to the Cauchy-Schwartz inequality, $\mathbb{E}[X]\mathbb{E}[Y]\geq\mathbb[\sqrt{XY}]^2$
for non-negative variables $X$ and $Y$. Let's choose $X = [\dfrac{|\nabla L(\theta_t)|^2}{4\lambda+2\tau\alpha\varphi + 4\tau\alpha|\nabla L(\theta_t)|}]$ and $Y = 4\lambda+2\tau\alpha\varphi + 4\tau\alpha|\nabla L(\theta_t)|$. This selection results in

\begin{equation}
\label{eq:40}
    \begin{aligned}
        &\mathbb{E}[\dfrac{\|\nabla L(\theta_t)\|^2}{4\lambda+4\tau\alpha\varphi + 4\tau\alpha\|\nabla L(\theta_t)\|}] \\&\geq \dfrac{\mathbb{E}[\|\nabla L(\theta_t)\|]^2}{4\lambda+4\tau\alpha\varphi + 4\tau\alpha\mathbb{E}[\|\nabla L(\theta_t)\|]} \\& \geq \dfrac{\mathbb{E}[\|\nabla L(\theta_t)\|]^2}{2\max\{4\lambda+4\tau\alpha\varphi, 4\tau\alpha\mathbb{E}[\|\nabla L(\theta_t)\|]\}} \\&= \min\{\dfrac{\mathbb{E}[\|\nabla L(\theta_t)\|]^2}{8\lambda+8\tau\alpha\varphi}, \dfrac{\mathbb{E}[\|\nabla L(\theta_t)\|]}{8\tau\alpha} \}.
    \end{aligned}
\end{equation}

Inserting \ref{eq:40} into \ref{eq:39} results in

\begin{equation}
\label{eq:41}
    \begin{aligned}
&\mathbb{E}[L(\theta_{t+1})] \\&\leq \mathbb{E}[L(\theta_t)] - \dfrac{1}{800}\min\{\dfrac{\mathbb{E}[\|\nabla L(\theta_t)\|]^2}{\lambda+\tau\alpha\varphi}, \dfrac{\mathbb{E}[\|\nabla L(\theta_t)\|]}{\tau\alpha} \}
 \\& + \varphi^2(\dfrac{7}{828\lambda B}+\dfrac{\alpha^2\lambda}{6})+\dfrac{\widehat{\varphi}^2/\mathcal{D}_q}{828\lambda B}+\dfrac{\widehat{\varphi}^2/\mathcal{D}_{s}}{600\lambda}.
    \end{aligned}
\end{equation}

Suppose \ref{eq:lem2} does not hold at iteration $t$. In that case, we have

\begin{equation}
\begin{aligned}
    &\mathbb{E}[\|\nabla L(\theta_t)\|] \\&\geq \max\{\sqrt{14(1+\dfrac{\tau\alpha}{\lambda}\varphi)(\varphi^2(\dfrac{1}{T_k}+20\alpha^2\lambda^2)+\dfrac{\widehat{\varphi}^2}{B\mathcal{D}_q}+\dfrac{\widehat{\varphi}^2}{\mathcal{D}_{s}})} \\& ,\dfrac{14\tau\alpha}{\lambda}(\alpha^2(\dfrac{1}{T_k}+20\alpha^2\lambda^2)+\dfrac{\widehat{\varphi}^2}{B\mathcal{D}_q}+\dfrac{\widehat{\varphi}^2}{B\mathcal{D}_{s}}) \}.
\end{aligned}
\end{equation}

This leads to

\begin{equation}
    \begin{aligned}
&\dfrac{1}{1600}\min\{\dfrac{\mathbb{E}[\|\nabla L(\theta_t)\|]^2}{\lambda+\tau\alpha\varphi}, \dfrac{\mathbb{E}[\|\nabla L(\theta_t)\|]}{\tau\alpha} \} \\& \geq \dfrac{1}{1600\lambda} (\varphi^2(\dfrac{1}{T_k}+20\alpha^2\lambda^2)+\dfrac{\widehat{\varphi}^2}{B\mathcal{D}_q}+\dfrac{\widehat{\varphi}^2}{\mathcal{D}_{s}}) \\& \geq \varphi^2(\dfrac{7}{828\lambda B}+\dfrac{\alpha^2\lambda}{6})+\dfrac{\widehat{\varphi}^2/\mathcal{D}_q}{828\lambda B}+\dfrac{\widehat{\varphi}^2/\mathcal{D}_{s}}{600\lambda}.
    \end{aligned}
\end{equation}

Therefore, utilizing \ref{eq:41}, we get

\begin{equation}
\label{eq:42}
    \begin{aligned}
&\mathbb{E}[L(\theta_{t+1})] \\&\leq \mathbb{E}[L(\theta_t)] - \dfrac{1}{1600}\min\{\dfrac{\mathbb{E}[\|\nabla L(\theta_t)\|]^2}{\lambda+\tau\alpha\varphi}, \dfrac{\mathbb{E}[\|\nabla L(\theta_t)\|]}{\tau\alpha} \} \\&\leq \mathbb{E}[L(\theta_t)]-\dfrac{1}{1600\lambda} (\varphi^2(\dfrac{1}{T_k}+20\alpha^2\lambda^2)+\dfrac{\widehat{\varphi}^2}{B\mathcal{D}_q}+\dfrac{\widehat{\varphi}^2}{\mathcal{D}_{s}}).
    \end{aligned}
\end{equation}

Taking as given that $\mathbb{E}[\nabla L(\theta_{t})] \geq \mu$, we also recognise that \ref{eq:lem2} does not hold at iteration $t$. This suggests

\begin{equation}
\label{eq:43}
    \begin{aligned}
&\mathbb{E}[L(\theta_{t+1})] \leq \mathbb{E}[L(\theta_t)] - \dfrac{1}{1600}\min\{\dfrac{\mu^2}{\lambda+\tau\alpha\varphi}, \dfrac{\mu}{\tau\alpha} \} \\&\leq \mathbb{E}[L(\theta_t)]-\dfrac{1}{1600\lambda} (\varphi^2(\dfrac{1}{T_k}+20\alpha^2\lambda^2)+\dfrac{\widehat{\varphi}^2}{B\mathcal{D}_q}+\dfrac{\widehat{\varphi}^2}{\mathcal{D}_{s}}).
    \end{aligned}
\end{equation}

\balance
This result shows that the objective function value declines by a constant value in expectation if the condition in \ref{eq:41} is not satisfied. By summing both sides of \ref{eq:43} from $0$ to $T - 1$, we can determine that, assuming that this condition does not apply for all iterations $0, \ldots, T - 1$,

\begin{equation}
    \sum_{t=0}^{T-1}\mathbb{E}[L(\theta_t+1)] \leq \sum_{t=0}^{T-1}\mathbb{E}[L(\theta_t)] -
    \sum_{t=0}^{T-1}\dfrac{1}{1600}\dfrac{\mu^2}{\lambda+\tau\alpha(\varphi+\mu)},
\end{equation}

which means

\begin{equation}
    \mathbb{E}[L(\theta_T)] \leq \mathbb{E}[L(\theta_0)] -
    \dfrac{T}{1600}\dfrac{\mu^2}{\lambda+\tau\alpha(\varphi+\mu)}.
\end{equation}

Hence 

\begin{equation}
\begin{aligned}
    T &\leq (\mathbb{E}[L(\theta_0)]  - \mathbb{E}[L(\theta_T)])1600\dfrac{\lambda+\tau\alpha(\varphi+\mu)}{\mu^2}
    \\& \leq (L(\theta_0)  - L(\theta^*))1600\dfrac{\lambda+\tau\alpha(\varphi+\mu)}{\mu^2}
\end{aligned}
\end{equation}

This argument shows that the time $T$ cannot exceed $(L(\theta_0) - L(\theta^))1600\dfrac{\lambda+\tau\alpha(\varphi+\mu)}{\mu^2}$ if the requirement in \ref{eq:lem2} is not satisfied for all $t$ from $0$ to $T - 1$. Consequently, following $(L(\theta_0) - L(\theta^))1600\dfrac{\lambda+\tau\alpha(\varphi+\mu)}{\mu^2}$ iterations, the proof is finished because at least one of the iterates produced by DP-Meta learning satisfies the requirement in \ref{eq:lem2}.

\end{proof}

\vfill

\end{document}